\pdfoutput=1
\documentclass[11pt]{article}
\usepackage[table]{xcolor}
\usepackage[]{ACL2023}
\usepackage{times}
\usepackage{latexsym}
\usepackage[T1]{fontenc}
\usepackage[utf8]{inputenc}

\usepackage{amsmath}
\usepackage{amssymb}
\usepackage{amsthm}
\usepackage{booktabs}
\usepackage[capitalize]{cleveref}
\usepackage{dsfont}
\usepackage{graphicx}
\usepackage{multirow}
\usepackage{floatrow}
\usepackage{array}
\newcolumntype{C}{@{\extracolsep{0.1cm}}c@{\extracolsep{0pt}}}%
\usepackage{wrapfig}
\usepackage{multicol}
\usepackage[english]{babel}
\usepackage{comment}
\usepackage{inconsolata}
\usepackage{caption}
\usepackage{subcaption}
\usepackage{algorithm}
\usepackage{algpseudocode}
\usepackage{algorithmicx}
\usepackage{todonotes}
\usepackage{enumitem}
\usepackage[table]{xcolor}
\usepackage{cuted}

\usepackage{amsmath,amsfonts,amssymb,amsthm,bm}
\usepackage{bbold}

\def\eqref#1{equation~\ref{#1}}

\def\1{\bm{1}}

\def\vx{{\bm{x}}}

\DeclareMathAlphabet{\mathsfit}{\encodingdefault}{\sfdefault}{m}{sl}
\SetMathAlphabet{\mathsfit}{bold}{\encodingdefault}{\sfdefault}{bx}{n}

\newtheorem{proposition}{Theorem}
\newtheorem{corollary}{Corollary}

\usepackage{tikz}
\usepackage{xspace}

\definecolor{cb-blue}       {RGB}{ 0, 109, 219}
\definecolor{cb-burgundy}   {RGB}{146,   0,   0}
\definecolor{cb-brown} {RGB}{150, 75,  0}
\definecolor{cb-green-lime} {RGB}{138, 226,  52}
\definecolor{cb-yellow}     {RGB}{253, 216, 53}

\newcommand{\GC}{\cellcolor{gray!15}}
\newcommand*\REDC{\tikz[baseline=(char.base)]{
        \node[draw=red, fill=red, text=white, shape=circle, minimum size=4mm, inner sep=0pt] (char)
        {\rule[-3pt]{0pt}{\dimexpr2ex+2pt}r};}\xspace}

\newcommand*\GREENC{\tikz[baseline=(char.base)]{
        \node[draw=cb-green-lime, fill=cb-green-lime, text=black, shape=circle, minimum size=4mm, inner sep=0pt] (char)
        {\rule[-3pt]{0pt}{\dimexpr2ex+2pt}g};}\xspace}
                
\newcommand*\BROWNC{\tikz[baseline=(char.base)]{
        \node[draw=cb-brown, fill=cb-brown, text=white, shape=circle, minimum size=4mm, inner sep=0pt] (char)
        {\rule[-3pt]{0pt}{\dimexpr2ex+2pt}b};}\xspace}

\theoremstyle{definition}
\newtheorem{definition}{Definition}
\newenvironment{example}[2]{
\vspace{3pt}
\noindent\textbf{Example #1}: \emph{#2}.
}
{
\vspace{3pt}
}

\theoremstyle{remark}

\newcommand{\inp}{\mathbf{x}}
\newcommand{\inpi}{x}
\newcommand{\inpspace}{\mathcal{X}}
\newcommand{\goodinpspace}{X}
\newcommand{\interp}{\mathcal{I}}
\newcommand{\vocab}{\Sigma}
\newcommand{\lexicon}{\mathcal{L}}
\newcommand{\typefn}{\tau}
\newcommand{\equivfn}{\epsilon}

\newcommand{\compfn}{\mathcal{C}}

\newcolumntype{P}[1]{>{\centering\arraybackslash}p{#1}}
\newcolumntype{M}[1]{>{\centering\arraybackslash}m{#1}}

\crefname{theorem}{theorem}{theorems}
\crefname{section}{Sec.}{Secs.}
\crefname{proposition}{Theorem}{Theorems}

\definecolor{cb-blue}       {RGB}{ 0, 109, 219}
\definecolor{cb-burgundy}   {RGB}{146,   0,   0}
\definecolor{cb-green-lime} {RGB}{138, 226,  52}
\definecolor{cb-yellow}     {RGB}{253, 216, 53}
\definecolor{gray}{RGB}{87, 87, 87}
\definecolor{red}{RGB}{173, 35, 35}
\definecolor{blue}{RGB}{42, 75, 215}
\definecolor{green}{RGB}{29, 105, 20}
\definecolor{brown}{RGB}{129, 74, 25}
\definecolor{purple}{RGB}{129, 38, 192}
\definecolor{cyan}{RGB}{41, 208, 208}
\definecolor{yellow}{RGB}{189, 167, 0}
\definecolor{Red}{rgb}{0.68, 0.05, 0.0}
\definecolor{Blue}{rgb}{0.0, 0.0, 0.61}
\definecolor{Blue1}{RGB}{214, 235, 245}
\definecolor{Blue2}{RGB}{235, 245, 250}
\definecolor{lime}{RGB}{60,179,113}

\usepackage{microtype}

\usepackage{pmboxdraw}

\newcommand\boxSpace{\hspace{1.5em}}
\newcommand\sboxSpace{\hspace{0.4em}}

\newcolumntype{L}{>{\hspace*{-\tabcolsep}}c}
\crefformat{equation}{Eq.~#2#1#3}

\newcommand{\stderr}[1]{\scriptsize $\pm #1$}
\newcommand\SCAN{\textsc{Scan}\xspace}
\newcommand\COGS{\textsc{Cogs}\xspace}
\newcommand\ALCH{\textsc{alchemy}\xspace}
\newcommand\CLEVR{\textsc{Clevr}\xspace}
\newcommand\CoGenT{\textsc{CoGenT}\xspace}
\newcommand\ourmethod{\textsc{LexSym}\xspace}

\title{LexSym: Compositionality as Lexical Symmetry}

\author{Ekin Aky\"urek ~~~~ Jacob Andreas \\
  Massachusetts Institute of Technology \\
  \texttt{\{akyurek,jda\}@mit.edu}
  }
\date{}

\begin{document}
\maketitle
\begin{abstract}
In tasks like semantic parsing, instruction following, and question answering, standard deep networks fail to generalize compositionally from small datasets. 
Many existing approaches overcome this limitation with model architectures that enforce a compositional process of sentence interpretation.
In this paper, we present a domain-general and model-agnostic formulation of compositionality as a constraint on \emph{symmetries of data distributions} rather than models. Informally, we prove that whenever a task can be solved by a compositional model, there is a corresponding data augmentation scheme---a procedure for transforming examples into other well-formed examples---that imparts compositional inductive bias on \emph{any} model trained to solve the same task. We describe a procedure called \ourmethod that discovers these transformations automatically, then applies them to training data for ordinary neural sequence models. Unlike existing compositional data augmentation procedures, \ourmethod can be deployed agnostically across text, structured data, and even images. It matches or surpasses state-of-the-art, task-specific models on \COGS semantic parsing, \SCAN and \ALCH instruction following, and \CLEVR-\CoGenT visual question answering datasets.
\end{abstract}

\section{Introduction}
\label{sec:introduction}

A central challenge in natural language processing is the design of models and learning algorithms that are simultaneously \emph{flexible} enough to capture the variability of human language and \emph{structured} enough to generalize in predictable and human-like ways. 
One important source of structure is the \textbf{principle of compositionality}, which (in one formulation) 
states that sentence meanings can be computed from a \emph{lexicon} of word meanings and a set of \emph{composition rules} governing how meanings combine \cite{montague1970universal}.
A long line of language processing research has operationalized the principle of compositionality as a \textbf{constraint on model architectures}, via independence assumptions or parameter tying schemes that ensure a compositional process of sentence interpretation \citep{lewis1968syntax,andreas2016neural}.
Compositional models enjoy sample-efficient learning and strong generalization in tasks from machine translation to question answering \cite{mccoy2020does}.

\begin{figure*}[t!]
    \centering
    \includegraphics[width=0.97\linewidth]{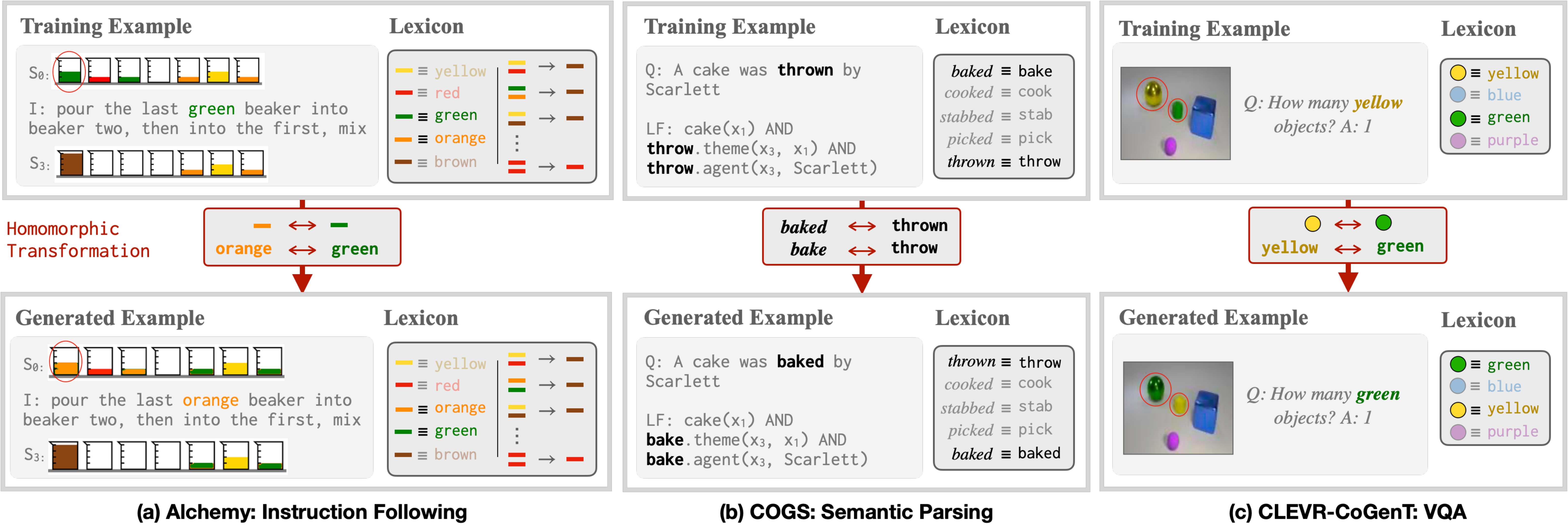}
    \vspace{-2em}
    \caption{We extract a lexicon that relates words to their meanings in each dataset. We then find \emph{homomorphic transformations} (\cref{sec:theory}) of this lexicon that, when applied to training examples, produce new, well-formed examples. (Note the changes in the generated examples)} %
    \label{fig:teaser}
    \strut
\end{figure*}

But much of human language is not (or at least not straightforwardly) compositional. Idioms, disfluencies, and context-sensitive meanings present major challenges to models in which all predictions must derive from a sequence of local composition operations. 
In recent years, more generic model architectures such as recurrent neural networks (RNNs) and transformers, with no explicit compositional scaffolding, 
have consistently outperformed compositional models in language processing tasks with natural data \cite{googlemt}. 
However, these models capture linguistic regularities only when trained on enormous amounts of data, and make surprising or problematic predictions when presented with novel word collocations or syntactic structures \cite{lake2018generalization}.

How can we train unstructured neural sequence models that generalize compositionally?
Recent work has introduced several \emph{compositional data augmentation} schemes: rule-based procedures or learned models that synthesize artificial training examples to promote generalization \cite[][\emph{inter alia}]{geca,shaw2020compositional,akyurek2020learning,zhang2022treemix}. While often effective, existing methods are specialized to specific data modalities or datasets. The conditions under which they succeed, and their relationships to the formal principle of compositionality, have remained unclear.

This paper presents a framework for understanding and improving such data-centric approaches to compositional modeling.
We first provide a mathematical characterization of the principle of compositionality as a 
\textbf{constraint on data distributions} rather than model architectures.
Intuitively, we show that whenever a language understanding task can be solved compositionally, that task's data distribution is guaranteed to exhibit specific \emph{symmetries}. These symmetries are functions that modify data points while preserving semantic acceptability.
\cref{fig:teaser}c gives an example of a symmetry in a visual question answering problem:
in any well-formed (image, question, answer) triple, swapping the words \emph{yellow} and \emph{green} and their associated pixel values yields a valid new triple. Such symmetries exist even in complex tasks like instruction following (\cref{fig:teaser}a), where they may depend not only on word-to-meaning mappings but relations \emph{between} meanings
(like the fact that red and green mix to produce brown).

Building on this formal link between compositionality and symmetry,
we introduce a procedure called \ourmethod that discovers symmetries automatically, then uses them to synthesize new training examples guaranteed to be correct and informative.
Crucially, \ourmethod does not require 
a complete compositional theory for a given problem domain---only a \emph{lexicon} of word meanings. These lexicons may themselves be automatically derived for most tasks. This makes \ourmethod very flexible: it requires little or no task-specific engineering, can be combined with any predictor, and
unlike other compositional data augmentation schemes does not require tree-structured or even sequential data.

Applied to ordinary neural sequence models,
\ourmethod outperforms state-of-the-art models on the
\CLEVR \CoGenT visual question answering benchmark \cite{johnson2017clevr} by a wide margin. 
\ourmethod is general, and matches or outperforms some specialized data augmentation schemes and models on the \COGS semantic parsing task \cite{kim2020cogs, kim2022uncontrolled}, and the \SCAN and \ALCH instruction following tasks \cite{lake2018generalization,long2016simpler}.

This paper thus offers two contributions: a theoretical contribution, in the form of a new lens on the principle of compositionality via symmetries of data distributions; and an empirical contribution, in the form of a data augmentation scheme that improves generalization on diverse language understanding tasks. The recent success of data augmentation approaches
highlight the fact that compositional inductive bias need not require compositional models. 
Our work formalizes and generalizes this ``data-centric'' account of compositionality.%
\footnote{Our implementations are released in \href{https://github.com/ekinakyurek/lexsym}{this https link}.}

\section{Background \& Approach}
\label{sec:background}

We begin with a discussion on the more general role of \emph{symmetry} in machine learning applications.

\begin{definition}\label{def:symmetry}
A \textbf{symmetry} of a set $X$ is a function $f$ satisfying:
\begin{equation}
\label{eq:symmetry}
    \{ f(\mathbf{x}) : \mathbf{x} \in X \} = X
\end{equation}
That is, applying $f$ to each element of $X$ leaves $X$ unchanged.
\end{definition}

A familiar example from computer vision is \emph{reflection symmetry}: in object recognition problems, image classes are generally invariant under reflection (a zebra seen in a mirror is still a zebra). The set of (image, class) pairs thus has as a symmetry the function $(\mathbf{x}, y) \mapsto (\texttt{reflect}(\mathbf{x}), y)$. In many domains, especially those (like computer vision and computational chemistry) that are constrained by physical laws, knowledge of the symmetries exhibited by a problem domain can dramatically reduce the difficulty of learning \cite{batzner20223, simeonov2022se}.

Past work has incorporated symmetry into machine learning problems in two ways. \textbf{Invariant and equivariant modeling} approaches structurally enforce symmetries via specialized architectures (improving generalization by decreasing the size of the hypothesis class; \citealp{cohen2016group}). \textbf{Data augmentation} approaches generate new training examples by applying known symmetries like reflections directly to training data (improving generalization by increasing dataset size; \citealp{shorten2019survey}). Data augmentation, the focus of this paper, is model-agnostic, and can be used in conjunction with pre-training 
while producing the same asymptotic effects as specialized model architectures \cite{chen2020group}.

The question this paper aims to answer is whether compositionality, like other domain-specific constraints, can be formalized in the language of symmetry.
We are not the first to consider this question: \citet{kiddon2015symmetry} define a theory of semantic equivalence in terms of symmetries of the set of natural language sentences, and \citet{Gordon2020Permutation} propose a model architecture for compositional semantic parsing via a symmetry that enforces \emph{permutation invariance} of lexicon entries. \ourmethod also derives symmetries from lexicons. It builds on past work by (1) characterizing the algebraic relationship between compositionality and symmetry, explaining the effectiveness of both \citet{Gordon2020Permutation}'s approach as well as other data augmentation schemes based on token and phrase substitution \citep{geca,wang2018switchout};  (2) discovering symmetries automatically, and (3) showing how to leverage them in a model- and modality-agnostic way. 
Additional related work is discussed in \cref{sec:otherrelated}.

\section{Compositionality as Lexical Symmetry}
\label{sec:theory}

Our main theoretical result, and the foundation of our modeling approach, can be stated as follows:
\emph{in any language understanding task that can be modeled compositionally, data for the task exhibits symmetries in the sense of \cref{def:symmetry}}.
We explain, formalize, and prove this statement below.

We consider tasks defined by a space of possible examples $\inpspace$, of which a subset of examples $\goodinpspace$ are \textbf{well-formed}.
We assume each example $\inp \in \inpspace$ is a discrete sequence $[\inpi_1, \ldots, \inpi_n]$, with $\inpi_i$ drawn from a vocabulary $\vocab$.
Finally, we assume that well-formedness can be computed by a a binary \textbf{interpretation function} $\interp : \inpspace \to \{0, 1\}$ with $\interp(\inp) = 1$ iff $\inp \in \goodinpspace$. A wide variety of language understanding problems, from very simple to very complex, may be defined in this way:

\begin{example}{1a}{Arithmetic Language Modeling}
Examples $\inp$ are true sentences of the form \emph{\underline{a} plus \underline{b} is \underline{c}}, where \emph{\underline{a}}, \emph{\underline{b}} and \emph{\underline{c}} are numbers: $\interp($\emph{one plus two is three}$) = 1$ but $\interp($\emph{two plus two is five}$) = 0$.
\end{example}

\newcommand{\inl}{\inp_\textrm{NL}}
\newcommand{\ilf}{\inp_\textrm{LF}}
\begin{example}{1b}{Semantic Parsing}
Examples $\inp$ are pairs $(\inl, \ilf)$, where $\inl$ is an sentence, $\ilf$ is a logical form, and $\interp(\inl, \ilf) = 1$ iff $\ilf$ represents a possible meaning of $\inl$ (\cref{fig:teaser}b).
\end{example}

\begin{example}{1c}{Visual Question Answering}
\newcommand{\iques}{\inp_\textrm{Q}}
\newcommand{\iim}{\inp_\textrm{I}}
\newcommand{\ians}{\inp_\textrm{A}}
Examples $\inp$ are triples $(\iques, \iim, \ians)$, where $\iques$ is a question, $\iim$ is a (rasterized) image, $\ians$ is an answer, and $\interp(\iques, \iim, \ians) = 1$ iff $\ians$ is the answer to $\iques$ in $\iim$ (\cref{fig:teaser}c).
\end{example}

\noindent
Notice that the vocabulary $\vocab$ contains not just natural language words, but other kinds of data: logical symbols (1b) or even image patches (1c).

``Language understanding'' in each of these tasks is encapsulated by the function $\interp$.
What does it mean for $\interp$ to be \emph{compositional}?
Under most definitions, a compositional  language understanding procedure should factorize into a lexicon, which captures meanings of words, and a composition procedure, which derives example-level interpretations from these meanings. 
We model word meanings in terms of \emph{relations} between items in $\vocab$.
In arithmetic, to know the meaning of the word \emph{five} is to know that it is a number, less than \emph{seven}, the successor of \emph{four}, etc. In semantic parsing, the meaning of the word \emph{cat} is encapsulated by the fact that it is of the same type as \emph{dog}, and translatable into the logical symbol \texttt{cat$'$}. We model this notion of word meaning by equipping $\vocab$ with extra structure describing these relations:

\begin{definition}\label{def:lex-alg}
A \textbf{lexical algebra} is a collection of relations $r_1, \ldots, r_n$ between vocabulary items, where each $r : \Sigma^p \to \{0, 1\}$. A lexical algebra can represent type information, like ``\emph{dog} is a noun'', as a unary relation; semantic correspondence, like ``\emph{sings} maps to \texttt{sing$'$}'', as a binary relation; and richer semantic knowledge, like ``\emph{three} is the sum of \emph{one} and \emph{two}'', with higher-order relations.
\end{definition}

\begin{figure}[t]
\centering
\includegraphics[width=0.95\columnwidth]{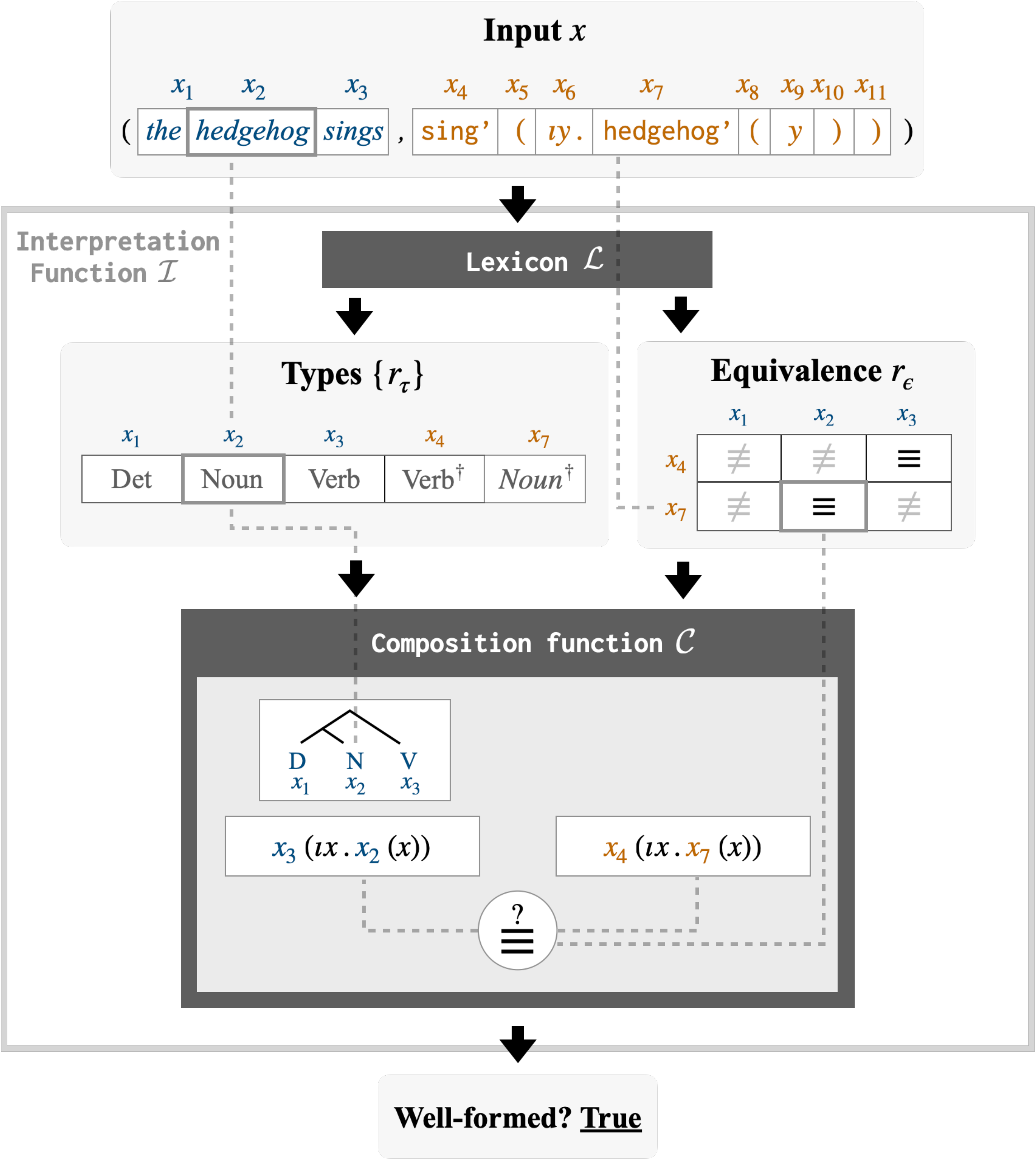}
\vspace{-.5em}
\caption{Idealized compositional semantic parser following \cref{def:abstraction}. A (sentence, logical form) pair is translated into a \emph{lexical representation} containing information about each word's type and meaning. We then determine whether the sentence evaluates to the logical form using \emph{only} the type and semantic correspondence matrices, using types to assign the sentence an abstract logical form, and correspondences to determine whether it matches the target.
\vspace{-1em}
}
\label{fig:semantic-parser}
\end{figure}

We may then represent individual examples in purely relational terms:

\begin{definition}\label{def:abstraction}
Denote the \textbf{lexical representation} $\lexicon(\inp) = (R_1(\inp), \ldots, R_n(\inp))$. $R(\inp)$ is an order-$p$ tensor whose $(i, \ldots, j)^{th}$ entry is equal to $r(\inpi_i, \ldots, \inpi_j)$. (If $r$ is a binary relation, $R(\inp)$ is an $|\inp| \times |\inp|$ matrix and $R(\inp)_{ij}$ specifies whether $r$ holds between $\inpi_i$ and $\inpi_j$.) See \cref{fig:semantic-parser} for examples.
\end{definition}

Finally, we use this relational representation to define compositionality of interpretation functions:

\begin{definition}\label{def:compositional}
$\goodinpspace$ is \textbf{$\lexicon$-compositional} if $\interp(\inp) = \compfn(\lexicon(\inp))$ for some \textbf{composition procedure} $\compfn$. In other words, $\goodinpspace$ is compositional if it compute the well-formedness of $\inp$ from word-level meanings and a generic composition procedure.\footnote{Every $\interp$ is trivially $\lexicon$-compositional with respect to an $\lexicon$ that assigns every vocabulary item to a unique unary relation. 
}
\end{definition}

This definition makes no assumptions about $\compfn$ beyond the fact that it can be defined purely in terms of $\lexicon(\inp)$. It can be applied to many tasks:

\begin{example}{2a}{Arithmetic Language Modeling}
Define $r_1$ to be the ternary relation $(a, b, c) \mapsto \mathds{1}_{[a + b = c]}$. Then $\compfn$ takes an example and checks whether the index corresponding to its three number words is true in $R_1$.
\end{example}

\begin{example}{2b}{Semantic Parsing}
A sketch of a semantic parser factorizable into a lexicon and an abstract composition function is depicted in \cref{fig:semantic-parser}.
As a real-world example, in the factored CCG semantic parser of \citet{kwiatkowski2011lexical}, words are assigned types and logical forms via a lexicon. These logical fragments are then composed by a parsing algorithm that depends only their types.
\end{example}

\begin{example}{2c}{Natural Language Inference}
\citet{maccartney2014natural}'s Natural Logic framework provides a procedure for determining entailment relations between sentences via a set of sentence rewriting operations that use only word-level information about entailment relations.
\end{example}

Under \cref{def:compositional}, a sentence interpretation procedure is compositional if the meaning of a sentence can be derived in a generic way ($\compfn$) from the meanings of its lexical items ($\mathcal{L}$).\footnote{As shown in Example 2b, it can be used to implement a language-to-logical form mapping, and thus generalizes the Montagovian definition of compositionality as a homomorphism from sentences to meanings \citep{montague1970english}.}
We remark, finally, that the parsing procedure depicted in \cref{fig:semantic-parser} is an idealization used to \emph{motivate} our approach; our experiments use more flexible models.

We are now ready to describe how, for compositional $\interp$, structure in $\lexicon$ translates into structure in the set of well-formed examples $\goodinpspace$. 

\begin{definition}\label{def:lhom}
A function $f$ is a \textbf{homomorphism of ($\pmb{\vocab}, \pmb{\lexicon}$}) (an ``$\lexicon$-homomorphism'') if:
\begin{align}\label{eq:lhom}
    &\forall r \in \lexicon, ~ \forall \inpi_1\ldots\inpi_p \in \vocab: \nonumber \\
    &\quad r(\inpi_1, \ldots, \inpi_p) = r(f(\inpi_1), \ldots, f(\inpi_p))
\end{align}
$f$ ``preserves the structure'' of $\lexicon$, ensuring that pairwise relationships are preserved among symbols. 
\cref{fig:teaser} shows examples: in (c), for instance, the words \emph{yellow} and \emph{green} and the corresponding colors must be \emph{swapped} to satisfy \cref{eq:lhom}. 
\end{definition}

Finally, we may state our main result:

\begin{proposition}\label{prop:lhomsymm}
    If $\goodinpspace$ is $\lexicon$-compositional, $f$ is an $\lexicon$-homomorphism, and $\inp \in \goodinpspace$, then $f(\inp) = [f(\inpi_1), \ldots, f(\inpi_n)] \in \goodinpspace$. Thus every homomorphism of $\lexicon$ well-formed examples $\in \goodinpspace$.
\end{proposition}
\begin{proof}
From \cref{def:abstraction} and \ref{def:lhom}, $R_i(f(\inp)) = R_i(\inp) \,\, \forall i$. Then,
\begin{align*}
    \mathds{1}_{[f(\inp) \in \goodinpspace]} &= \interp(f(\inp)) \\
    &= \compfn(\lexicon(f(\inp))) \\
    &= \compfn(R_1(f(\inp)), \ldots, R_n(f(\inp))) \\
    &= \compfn(R_1(\inp), \ldots, R_n(\inp)) \\
    &= \interp(\inp) = \mathds{1}_{[\inp \in \goodinpspace]} \qedhere
\end{align*}
\end{proof}

\begin{corollary}\label{prop:symmetry}
With the additional constraint that $f$ is an $\lexicon$-\emph{isomorphism} (i.e., has an inverse), then $f$ is a symmetry of $\goodinpspace$ in the sense of \cref{eq:symmetry}.
\end{corollary}
Here it suffices to show that the preimage of every $\inp \in \goodinpspace$ is also in $\goodinpspace$; the proof is the same as \cref{prop:lhomsymm} with $f^{-1}$ in place of $f$.

Despite their simplicity, \cref{prop:lhomsymm} and its corollary have an important consequence: if we can identify candidate entries in $\lexicon$, \emph{even if  ~$\compfn$ is unknown}, we can construct new examples $\inp \in \goodinpspace$ that respect, and provide evidence for, the compositional structure of $\goodinpspace$. 
There is an intriguing (if inexact) structural similarity between \cref{prop:symmetry} and Noether's theorem \citep{Noether1918}, which establishes an equivalence between symmetries of physical systems and their conserved quantities. Here, such symmetries imply constraints not on conservation laws but interpretation functions.

\section{\ourmethod: Data Augmentation with $\pmb{\lexicon}$-homomorphisms}
\label{sec:method}

Given a lexicon describing symbols and their relations, we have shown
how to turn homomorphisms of the lexicon into 
transformations of a dataset. 
Each such function $f$ that takes an example $\inp$ as input, replaces each token $\inpi_i \in \inp$ with a new one, and returns a well-formed example $\inp'$ as output. Every $\lexicon$-homomorphism may thus be viewed as a recipe for \emph{synthesizing training examples} from a small initial training set \cite{japkowicz2000learning}. However, to make this a practical modeling tool, we need some way of constucting $\lexicon$-homomorphisms for a task of interest. Below, we describe how to do so automatically: first, starting with only a task-specific lexicon $\mathcal{L}$ (\cref{ssec:from-lex}); next, starting with only a dataset and no initial lexicon (\cref{ssec:from-data}). 
We term the resulting approach
\textbf{\ourmethod}.

\subsection{Deriving Homomorphisms from Lexicons}
\label{ssec:from-lex}

Even in complex sequence modeling problems, useful lexicons are often simple enough that they can be specified by hand
\citep{jones2012semantic,Gordon2020Permutation}. Given a pre-specified algebraic $\lexicon$, there is a straightforward procedure for generating the associated symmetries by enumerating all functions $\vocab \to \vocab$ and testing which ones satisfy \cref{eq:lhom}. (See \cref{alg:lhom} in \cref{app:alg}.)
This algorithm is inefficient, but simple and practical for small $|\lexicon|$.

\subsection{Deriving Lexicons from Datasets}
\label{ssec:from-data}

For some tasks, it may be difficult to manually specify an algebraic lexicon. We next describe how to infer one automatically.
We focus on an important and extremely common class of language understanding problems with special structure. In \emph{semantic parsing} and \emph{instruction following}, examples $\inp$ consist of (input, output) pairs in which
inputs are sentences, outputs are meaning representations, and word meaning is characterized by a lexicon with two components. First, a set of unary \textbf{type predicates} $\{r_\typefn\}$ that assign words to types (like \textsc{entity} in semantic parsing).
Second, a \textbf{semantic correspondence relation} $r_\equivfn$ that specifies which actions or logical symbols can be derived from words (like \emph{sings} $\to$ \texttt{sing}$'$). With $n$ types, the lexicon required for these problems is $\lexicon = (r_{\typefn_1}, \ldots, r_{\typefn_n},  r_{\equivfn})$, which we abbreviate $(\{r_{\typefn_k}\},  r_{\equivfn})$ below.
We now show how to improve upon the procedure in \cref{ssec:from-lex}
by deriving $\lexicon$ from data
and sampling $\lexicon$-homomorphisms in constant time.

\paragraph{Learning $\pmb{\lexicon}$}

We build on past work noting that dictionaries of semantic correspondences can be constructed using alignment algorithms \citep{brown1993mathematics}.
Given an input $\inp$ consisting of a pair $(\inp_\text{text}, \inp_\text{meaning})$,
we use existing algorithms to align tokens in individual training examples. Finally, we identify the most frequently occurring alignments and add these to the semantic correspondence relation.
We may similarly use existing procedures to infer types by deriving them from part-of-speech tags or distributional patterns. %
See \cref{app:lexlearn} for details of the alignment and type inference algorithms used in our experiments.
These algorithms produce lexicons with three properties that are useful for the sampling scheme we describe next: types are \emph{disjoint}, and semantic correspondences are \emph{one-to-many} and \emph{type-preserving} (if two words are of the same type, so are their translations).

\paragraph{Sampling $\pmb{\lexicon}$-homomorphisms}
Once we have identified types and semantic correspondences, sampling $\lexicon$-homomorphisms is straightforward:

\begin{proposition}\label{prop:fswap}
Let $\inpi_i$ and $\inpi_j \in \vocab$ have the same type $r_\typefn(x_i) = r_\typefn(x_j) = 1$. For convenience, let $E_i=\{\inpi : r_\equivfn(\inpi_i, \inpi) = 1\}$ denote possible translations of $x_i$. The $f$ is an $\lexicon$-homomorphism:
\begin{equation}
\label{eq:fswap}
    f(\inpi) = \begin{cases}
        \inpi_j & \textrm{if } \inpi = \inpi_i \\
        \inpi_i & \textrm{if } \inpi = \inpi_j \\
        x' \in E_j & \textrm{if } \inpi \in E_i\\
        x' \in E_i & \textrm{if } \inpi \in E_j\\
        \inpi & \textrm{otherwise}
    \end{cases}
\end{equation}
\label{prop:homex}
\end{proposition}
\noindent
Proof is given in \cref{app:fswap}.
\cref{prop:fswap} yields an intuitive data augmentation procedure: select two (input, output) pairs of the same type, and \emph{swap} them and any of their meanings wherever they occur.
\cref{fig:teaser}b shows an example.
\cref{eq:fswap} is related to data augmentation schemes described by \citet{geca} and  \citet{liu2021counterfactual}, which synchronously \emph{substitute} words or phrases (equivalent to removing cases 2 and 4). Unlike \ourmethod, these methods cannot guarantee correctness:
in \cref{fig:teaser}c, substituting \emph{green} in place of \emph{yellow} yields an image with two green objects and an incorrect answer.

\section{Experiments}
\begin{table*}[t]
\centering
\resizebox{0.95\textwidth}{!}{%
\begin{tabular}{llllll}
\toprule
\bf Model & \bf \ALCH  & \bf \SCAN (\emph{jump}) & \bf \SCAN (\emph{around right}) & \bf \COGS & \bf \COGS (nonce) \\	
\midrule

\multicolumn{1}{l}{\textbf{Previous Work on COGS \& SCAN}}   \\
\boxSpace  GECA \cite{geca} & --& 99.94  \stderr{0.10}& 98.50  \stderr{1.90}& 47.74  \stderr{4.52}& --\\
\boxSpace LeAR \cite{liu2021learning} & --& --& --& 97.70  \stderr{0.70}& --\\
\boxSpace LexLSTM \cite{akyurek2021lexicon} & 36.80  \stderr{1.96}& 99.14  \stderr{1.55}& 88.41  \stderr{7.35}& 82.17  \stderr{0.72}&  81.40 \stderr{0.40} \\

\midrule
\multicolumn{1}{l}{\textbf{No Pre-training}} &  \\

\boxSpace LSTM 	& 41.72  \stderr{1.15}& \textcolor{white}{00}0.41  \stderr{0.34}& \textcolor{white}{0}8.65  \stderr{4.52}& 61.13  \stderr{4.12}& 61.13 \stderr{ 4.12}\\
\boxSpace \boxSpace+ Substitute \citep[e.g.][]{liu2021counterfactual} 	& 40.52  \stderr{0.84}& \textcolor{white}{0}99.95  \stderr{0.10}& 99.17  \stderr{0.93}& 81.99  \stderr{0.50}&  77.62  \stderr{0.78}\\
\boxSpace \boxSpace\GC + \ourmethod & \GC 45.85  \stderr{2.00}& \GC 100.00  \stderr{0}& \GC 99.51  \stderr{0.48}&\GC 81.86  \stderr{0.90}&\GC 77.25  \stderr{0.34} \\
\midrule
\multicolumn{1}{l}{\textbf{Language Pre-training}}\\
\boxSpace T5 	& 84.95  \stderr{0.44}& 93.60  \stderr{0}& 38.40  \stderr{0.90}& 83.30  \stderr{0.10}& 64.20  \stderr{2.00}\\
\boxSpace  \boxSpace   +CSL-Aug* \cite{qiu2021improving} & --& 99.70  \stderr{0}& --&  99.50  \stderr{0}& --\\
\boxSpace  \boxSpace  \GC +\ourmethod  & \GC 85.48  \stderr{0.16}&  \GC 99.96 \stderr{0.03}&  \GC 97.29 \stderr{2.16} & \GC 83.62  \stderr{0.27}& \GC 76.74  \stderr{2.23} \\
\bottomrule
\end{tabular}
}
\caption{Results on semantic parsing and instruction following. We provide mean and standard deviations over 5 random seeds. \ourmethod improves significantly over baselines, with and without large-scale pretraining. \\
{\footnotesize *Uses a customized formal representation.}
}
\label{tab:generalization_comparison}
\end{table*}

\begin{table}[t]
\caption{Exact match accuries on the \CLEVR and \CLEVR-\CoGenT validation sets. Results are averaged over 4 seeds. 
We obtain state-of-the-art results after applying \ourmethod to a (non-pretrained) sequence model. \ourmethod also yields higher accuracies than synchronous token \emph{substitution}. (A detailed breakdown by question category is presented in \cref{tab:clevr_cogent_detailed}). %
}
\label{tab:clevr_cogent}
\centering
\resizebox{0.95\linewidth}{!}{%
\begin{tabular}{lllllllll}
\toprule
  & \multicolumn{1}{c}{\bf \CoGenT} & \bf \CLEVR  \\ 
 \midrule
  \multicolumn{1}{l}{\textbf{Visual Pre-training}}  \\

\boxSpace Human \cite{johnson2017clevr} &--& 92.6\\ 
\boxSpace Film \cite{perez2018film} & 78.8 & 97.7 \\ 
\boxSpace S-MAC \cite{marois2018transfer} & 78.7 & 98.9\\ 
\boxSpace NSVQA \cite{yi2018neural} & 63.9 &  99.7 \\
\midrule 

\multicolumn{1}{l}{\bf Seq2Seq Baselines}  \\
T5 & 79.7 & -- \\
LexLSTM & 62.1 & -- \\
 \midrule 
 
\multicolumn{1}{l}{\bf No Pre-Praining}  \\

\boxSpace VQATransformer   & 73.3 \stderr{1.0}   & 93.6 \stderr{0.5}  \\
\boxSpace + Substitute \cite[e.g.][]{liu2021counterfactual} & 84.4 \stderr{0.7}     & 90.8 \stderr{0.3} \\
\boxSpace\GC + LexSym  &\GC 85.9 \stderr{0.9} & \GC 92.0 \stderr{0.9}  \\
 \bottomrule
\end{tabular}
}
\end{table}

Our experiments aim to evaluate whether \ourmethod can
improve compositional generalization in downstream models.
The main goal of these experiments is to evaluate
\emph{generality} across tasks and data modalities. Evaluation focuses on three diverse classes of language understanding problems: complex, context-dependent computations (\cref{subsec:complex}), large, automatically
derived lexicons (\cref{subsec:learnedeq}), and multi-modal data (\cref{subsec:multimodal}).

\subsection{Complex computations}\label{subsec:complex}
We first test \ourmethod on the \ALCH task from the \textsc{scone} benchmark \cite{long2016simpler}---a problem involving a complex sentence interpretation procedure that makes it challenging to apply existing data augmentation schemes.
\paragraph{Data} In \ALCH (\cref{fig:teaser}a), models must execute a sequence of human-written English instructions $\vx^{1:\text{N}}_{\text{ins}}$, on an initial state $\vx^0_{\textrm{state}}$ consisting of beakers of colored liquids (textually represented as sequence of symbols ``1: \GREENC\GREENC, 2: ...''), to predict the final state $\vx^N_{\textrm{state}}$. Initial and final states are encoded as sequences of color tokens. Predicting final states requires both grounding colors in state variables (brown $\to$ \BROWNC, red $\to$ \GREENC) and modeling what happens when colors are combined
(e.g.\ mixing \GREENC and \REDC yields \BROWNC).
\paragraph{\ourmethod} We manually construct a lexicon to showcase how to inject prior knowledge into \ourmethod. 
We encode word meaning in two relations:
a semantic equivalence relation between color words and colors:
\begin{equation*}
r_\equivfn(c_1, c_2) =  \begin{cases}
    1 \quad c_1=\text{brown}, & c_2=\BROWNC \\
    1 \quad c_1=\text{red}, & c_2=\REDC \\
    1 \quad c_1=\text{green}, & c_2=\GREENC \\
    \vdots \\
    0 \quad \text{otherwise}
\end{cases}
\end{equation*}
and a ternary relation that encodes the result of mixing colors:\footnote{In \ALCH, mixing non-identical colors produces \BROWNC.} 
\begin{equation*}
r_{\texttt{mix}}(c_1, c_2, c_3) = \begin{cases}
    1 \quad c_1 = c_2 = c_3 \\
    1 \quad c_1 \neq c_2 \land c_3 = \BROWNC \\
    0 \quad \textrm{otherwise}
\end{cases}
\end{equation*}

\noindent
Together, $(r_\equivfn, r_\texttt{mix}, \{r_{\typefn_k}\})$, where $\{r_{\typefn_k}\}$ 
assigns different types to color words, colors, and remaining tokens.
The homomorphic transformations of this lexicon exchange color words and colors but preserve mixing relations. %

\paragraph{Models and Training}
We train an LSTM \cite{hochreiter1997long} and fine-tune a T5 transformer \cite{raffel2019exploring} on the sequence-to-sequence prediction problem ($\vx^{1:\text{N}}_{\textrm{ins}}, \vx^0_{\textrm{state}}$) $\to$ $\vx^N_{\textrm{state}}$
Training details may be found in \cref{app:impdetails}.  We compare these baseline models to their \ourmethod-augmented versions 
as well as the existing compositional data augmentation scheme of \citet{liu2021counterfactual}.

\paragraph{Results}
See \cref{tab:generalization_comparison}.
LSTM+\ourmethod improves substantially over an LSTM. 
Preserving the homomorphism condition in \cref{eq:lhom} is extremely important: the procedure of \citet{liu2021counterfactual}, which naively substitutes aligned color pairs, actually \emph{hurts} performance.
Pre-trained models achieve strong initial results; combining pre-training with \ourmethod gives additional improvements.

\subsection{Learned lexicons}\label{subsec:learnedeq}

We next show that for more conventional sequence-to-sequence problems, we may apply \ourmethod with automatically derived lexicons.

\paragraph{Data} We study two standard compositional generalization benchmarks: the \SCAN \cite{lake2018generalization} instruction following  and \COGS \citep[\cref{fig:teaser}b]{kim2020cogs} semantic parsing datasets. 
\SCAN consists of simple instruction following tasks in which strings are translated into sequences of actions.
We focus on the \emph{jump} split, which measures models' ability to compose words that only appeared in isolation during training, and the \emph{around right} split, which measures generalization to novel collocations.
The \COGS dataset tests compositional generalization in semantic parsing. The dataset includes English (sentence, logical form) pairs, with systematic differences between train and test set sentence structure. We include a variant containing nonce words 
\cite{kim2022uncontrolled} to disentangle general compositional skills from lexical knowledge acquired during pre-training. See \cref{app:data} for dataset statistics.

\paragraph{\ourmethod}
We use automatic lexicon extraction to find semantic correspondence relations ($r_\equivfn$) and types ($\{r_{\typefn_k}\}$) as described in \cref{app:lexlearn}. 
Next, we apply swap-based augmentation (\cref{eq:fswap}).

\paragraph{Models} We use the same models as  \cref{subsec:complex}, along with a strong semi-structured model, LeAR \cite{liu2021learning} tailored for \COGS, and another substitution based augmentation \cite{geca} tailored for SCAN. Following \citet{akyurek2021lexicon}, we equip the LSTM for \COGS with a copy mechanism as it achieves significantly better results than \citet{kim2020cogs}'s baseline.

\paragraph{Results}
On \SCAN, \ourmethod obtains near-perfect accuracy in both \emph{jump} and \emph{around right} splits.
On the original \COGS  datasets, \ourmethod substantially outperforms the LSTM model and GECA augmentation,
and is comparable to a neural sequence model specialized for lexical generalization (LexLSTM).
Stronger results can be achieved with models specifically tailored toward semantic parsing tasks (LeAR).
 In both tasks, \ourmethod also improves upon large-scale pre-training.

\subsection{Multi-modal data}\label{subsec:multimodal}

Finally, we combine learned lexicons with non-sequential data to advance the state of the art on a long-standing visual question answering challenge.

\paragraph{Data} The \CLEVR dataset \citep[\cref{fig:teaser}c]{johnson2017clevr} contains English-language questions about generated 3D scenes containing multiple objects. Questions involve complex computational operations including quantification, comparison, and spatial reasoning. \CLEVR has been a popular testbed for evaluating composition in visual question answering models. Our main experiment uses the \CoGenT split of the dataset, which focuses on compositional generalization.
In the \CLEVR-\CoGenT training set (Split A), which contains roughly $700K$ (question, image, answer) triples, all cubes are gray, blue, brown or yellow, while all cylinders are red, green, purple or cyan. In the test set (validation set of Split B), these are reversed.
\paragraph{\ourmethod}
In VQA and other multi-modal tasks, part of the input is continuous (e.g.\ images and videos). Recent work has shown that it is possible to \emph{learn} high-quality discrete representations of continuous input data. For example, in the VQ-VAE model of \citet{oord2017neural}, a continuous image is transformed into a grid of categorical codes, with individual codes representing color, and in some cases materials and illumination (examples in \cref{tab:samplevqacogs}). We use this discretization procedure for our experiments (see \cref{app:vqvae} for details). 
We use the same algorithm as previous section to extract lexical relations.

\paragraph{Models}
Most prior work on visual question answering has used pre-trained convolutional networks to encode images, and recurrent networks to encode questions and generate answers. For experiments on \CLEVR, we use a simplified model in which both questions and images are mapped to answers by a transformer model, similarly to \citet{ramesh2021zero}. See  \cref{app:vqatransformer} for details.

Both \ourmethod augmentation and this VQATransformer model operate over sequences of discrete visual codes produced by a vector-quantized variational autoencoder. Once these discrete representations have been produced, we infer lexicons and perform data augmentation directly to these representations, without re-synthesizing images (though such synthesis is possible, as in \cref{tab:samplevqacogs}, to interpret model behavior).

The \CoGenT task is very different from the sequence modeling tasks discussed above: inputs contain many tokens, and the training set is orders of magnitude larger. GECA and CSL-Aug, which have a high polynomial dependence on sequence length, could not be applied as they fail to terminate within a reasonable amount of time.

\paragraph{Results}
In \cref{tab:clevr_cogent}, a transformer model with \ourmethod achieves state-of-the-art results on the \CLEVR-\CoGenT dataset, 
reducing errors by roughly 33\%
relative to the best existing system. 
\ourmethod also outperforms substitution 
 based-data augmentation \citep{liu2021counterfactual}, particularly on semantically complex utterances involving quantification (App.\ \cref{tab:clevr_cogent_detailed}). On the IID \CLEVR split, \ourmethod's performance is comparable to humans, and somewhat behind pre-trained models.

\section{Other Related Work}
\label{sec:otherrelated}

\paragraph{Lexicalized neural models}
Word-level alignments between input and output sequences were an essential feature of statistical phrase- and tree-based sequence models \cite{chiang2005hiero, Koehn2003StatisticalPT}. 
Neural scoring functions were sometimes integrated into these models \cite{misra2016neural}.
Neural models with attention \cite{bahdanau2014neural} do not require explicit alignment, though several pieces of past work have shown that incorporating explicit token-level correspondences improves generalization \cite{akyurek2021lexicon, Prabhu2020MakingAP, pham2018towards}. 
The semantic correspondence function in \cref{sec:method} plays the same role as the input--output dictionary in these methods, but \ourmethod as a whole is more general: it is not restricted to modeling sequence-to-sequence problems, and can infer and exploit correspondence relations between component of an example.
To the best of our knowledge, this paper is also the first to make use of token-level alignments in joint neural models of text and images. 

\paragraph{Compositionality in representation learning}
While we have focused on compositionality as a property of data distributions or interpretation functions, another line of work in machine learning and language evolution has studied compositionality as an emergent property of learned representations \cite{andreas2019measuring, resnick2019capacity, brighton2006understanding}. In settings where representational compositionality is desirable (e.g.\ to train communication protocols that can generalize to new states), \ourmethod might provide a tool for promoting it.

\paragraph{Equivariant Sequence Models}

As mentioned in \cref{sec:background}, our work builds on existing approaches that control generalization with specialized model architectures designed to be equivariant to permutations of a pre-specified lexicon (if $f(x_1 \cdots x_n) = y_1 \cdots y_m$ then $f(\pi(x_1) \cdots \pi(x_n)) = \pi(y_1) \cdots \pi(y_m)$ for a permutation $\pi$) \citep{Gordon2020Permutation,white-cotterell-2022-equivariant}.
\ourmethod differs from these approaches in three ways. First, \ourmethod is model-agnostic and compatible with pre-training. Second, \ourmethod is compatible with (and automatically derives transformations for) more complicated relations than input--output correspondences, making it possible to apply to tasks like \ALCH where such relations are important. Finally, \ourmethod gracefully handles (possibly noisy) learned lexicons, making it applicable to tasks like \CoGenT with complex or uninterpretable token mappings.

\paragraph{Data Augmentation}
Data augmentation approaches are widely used across machine learning application domains featuring known invariances of the data distribution \cite{japkowicz2000learning, data-recombination-copy, shaw2020compositional}.
Substitution-based schemes that replace words with synonyms, or synchronously replace words and their translations, are widely 
used for machine translation and general de-biasing \cite{liu2021counterfactual, wang2018switchout, wei2019eda}. %

\section{Limitations and Future Directions}
\label{sec:limitations}
While \cref{sec:theory} characterizes the effect of general $\lexicon$-homomorphisms, \ourmethod specifically produces
single-token swaps. 
In images represented as discrete symbol sequences, if a single symbol simultaneously encodes multiple visual features (e.g.\ color and texture), these features will remain entangled in synthesized examples. 
It will not exchange substructures larger than a single token, and thus will not synthesize examples longer than those already present in the training set \citep{lake2019human}.
This is because \ourmethod targets compositionality but not \emph{recursion}, which is also required to model the full range of human-like generalizations in sequence learning problems.

\ourmethod is also sensitive to the nature of the tokenization scheme itself. In morphologically rich languages, for example, \ourmethod may need to be applied not on top of words or segments, but instead canonicalized morphemes produced by learned morphological analyzers \citep{narasimhan2015unsupervised, bergmanis2017segmentation, cotterell2018joint} (analogous to the use of learned image patch representations rather than pixels in our VQA experiments).

Finally, \ourmethod does not induce some of the generalizations obtained other methods for improving compositional generalization, especially those that exploit extra structure (e.g.\ tree-shaped inputs and outputs) in the semantic parsing domain \citep[e.g.][]{liu2021learning}.
It might serve as a platform for future versions of those methods that offer greater generality and formal guarantees.

\section{Conclusion}
We have presented \ourmethod, a new data augmentation method that improves compositional generalization of neural models in multiple domains. \ourmethod is derived from a characterization of the principle of compositionality as a constraint on the symmetries of data distributions, and a procedure for automatically identifying these symmetries using token-level alignments.
Our results highlight the fact that many inductive biases targeted by specialized models in NLP can be alternatively, and often more flexibly, expressed as a hypothesis about the structure of the distribution to be modeled.

 \section*{Acknowledgements}

This work was supported by the MachineLearningApplications initiative at MIT CSAIL, the MIT--IBM Watson AI lab, and the National Science Foundation under grant CCF-2217064.
Computing resources were provided by a gift from NVIDIA through the NVAIL program and by the Lincoln Laboratory Supercloud.

\section*{Ethics Statement}
We do not anticipate any ethical issues associated with the techniques decribed in this paper.

\bibliography{anthology,acl2023.rebibed}

\begin{thebibliography}{61}
\expandafter\ifx\csname natexlab\endcsname\relax\def\natexlab#1{#1}\fi

\bibitem[{Aky{\"{u}}rek et~al.(2021)Aky{\"{u}}rek, Aky{\"{u}}rek, and
  Andreas}]{akyurek2020learning}
Ekin Aky{\"{u}}rek, Afra~Feyza Aky{\"{u}}rek, and Jacob Andreas. 2021.
\newblock \href {https://openreview.net/forum?id=PS3IMnScugk} {Learning to
  recombine and resample data for compositional generalization}.
\newblock In \emph{9th International Conference on Learning Representations,
  {ICLR} 2021, Virtual Event, Austria, May 3-7, 2021}. OpenReview.net.

\bibitem[{Akyurek and Andreas(2021)}]{akyurek2021lexicon}
Ekin Akyurek and Jacob Andreas. 2021.
\newblock \href {https://doi.org/10.18653/v1/2021.acl-long.382} {Lexicon
  learning for few shot sequence modeling}.
\newblock In \emph{Proceedings of the 59th Annual Meeting of the Association
  for Computational Linguistics and the 11th International Joint Conference on
  Natural Language Processing (Volume 1: Long Papers)}, pages 4934--4946,
  Online. Association for Computational Linguistics.

\bibitem[{Andreas(2019)}]{andreas2019measuring}
Jacob Andreas. 2019.
\newblock \href {https://openreview.net/forum?id=HJz05o0qK7} {Measuring
  compositionality in representation learning}.
\newblock In \emph{7th International Conference on Learning Representations,
  {ICLR} 2019, New Orleans, LA, USA, May 6-9, 2019}. OpenReview.net.

\bibitem[{Andreas(2020)}]{geca}
Jacob Andreas. 2020.
\newblock \href {https://doi.org/10.18653/v1/2020.acl-main.676} {Good-enough
  compositional data augmentation}.
\newblock In \emph{Proceedings of the 58th Annual Meeting of the Association
  for Computational Linguistics}, pages 7556--7566, Online. Association for
  Computational Linguistics.

\bibitem[{Andreas et~al.(2016)Andreas, Rohrbach, Darrell, and
  Klein}]{andreas2016neural}
Jacob Andreas, Marcus Rohrbach, Trevor Darrell, and Dan Klein. 2016.
\newblock \href {https://doi.org/10.1109/CVPR.2016.12} {Neural module
  networks}.
\newblock In \emph{2016 {IEEE} Conference on Computer Vision and Pattern
  Recognition, {CVPR} 2016, Las Vegas, NV, USA, June 27-30, 2016}, pages
  39--48. {IEEE} Computer Society.

\bibitem[{Bahdanau et~al.(2015)Bahdanau, Cho, and Bengio}]{bahdanau2014neural}
Dzmitry Bahdanau, Kyunghyun Cho, and Yoshua Bengio. 2015.
\newblock \href {http://arxiv.org/abs/1409.0473} {Neural machine translation by
  jointly learning to align and translate}.
\newblock In \emph{3rd International Conference on Learning Representations,
  {ICLR} 2015, San Diego, CA, USA, May 7-9, 2015, Conference Track
  Proceedings}.

\bibitem[{Batzner et~al.(2022)Batzner, Musaelian, Sun, Geiger, Mailoa,
  Kornbluth, Molinari, Smidt, and Kozinsky}]{batzner20223}
Simon Batzner, Albert Musaelian, Lixin Sun, Mario Geiger, Jonathan~P Mailoa,
  Mordechai Kornbluth, Nicola Molinari, Tess~E Smidt, and Boris Kozinsky. 2022.
\newblock E (3)-equivariant graph neural networks for data-efficient and
  accurate interatomic potentials.
\newblock \emph{Nature communications}, 13(1):1--11.

\bibitem[{Bergmanis and Goldwater(2017)}]{bergmanis2017segmentation}
Toms Bergmanis and Sharon Goldwater. 2017.
\newblock \href {https://aclanthology.org/E17-1032} {From segmentation to
  analyses: a probabilistic model for unsupervised morphology induction}.
\newblock In \emph{Proceedings of the 15th Conference of the {E}uropean Chapter
  of the Association for Computational Linguistics: Volume 1, Long Papers},
  pages 337--346, Valencia, Spain. Association for Computational Linguistics.

\bibitem[{Brighton and Kirby(2006)}]{brighton2006understanding}
Henry Brighton and Simon Kirby. 2006.
\newblock Understanding linguistic evolution by visualizing the emergence of
  topographic mappings.
\newblock \emph{Artificial life}, 12(2):229--242.

\bibitem[{Brown et~al.(1993)Brown, Della~Pietra, Della~Pietra, and
  Mercer}]{brown1993mathematics}
Peter~F. Brown, Stephen~A. Della~Pietra, Vincent~J. Della~Pietra, and Robert~L.
  Mercer. 1993.
\newblock \href {https://aclanthology.org/J93-2003} {The mathematics of
  statistical machine translation: Parameter estimation}.
\newblock \emph{Computational Linguistics}, 19(2):263--311.

\bibitem[{Chen et~al.(2020)Chen, Dobriban, and Lee}]{chen2020group}
Shuxiao Chen, Edgar Dobriban, and Jane~H. Lee. 2020.
\newblock \href
  {https://proceedings.neurips.cc/paper/2020/hash/f4573fc71c731d5c362f0d7860945b88-Abstract.html}
  {A group-theoretic framework for data augmentation}.
\newblock In \emph{Advances in Neural Information Processing Systems 33: Annual
  Conference on Neural Information Processing Systems 2020, NeurIPS 2020,
  December 6-12, 2020, virtual}.

\bibitem[{Chiang et~al.(2005)Chiang, Lopez, Madnani, Monz, Resnik, and
  Subotin}]{chiang2005hiero}
David Chiang, Adam Lopez, Nitin Madnani, Christof Monz, Philip Resnik, and
  Michael Subotin. 2005.
\newblock \href {https://aclanthology.org/H05-1098} {The {H}iero machine
  translation system: Extensions, evaluation, and analysis}.
\newblock In \emph{Proceedings of Human Language Technology Conference and
  Conference on Empirical Methods in Natural Language Processing}, pages
  779--786, Vancouver, British Columbia, Canada. Association for Computational
  Linguistics.

\bibitem[{Clark and Eyraud(2007)}]{clark2007polynomial}
Alexander Clark and R{\'e}mi Eyraud. 2007.
\newblock Polynomial identification in the limit of substitutable context-free
  languages.
\newblock \emph{Journal of Machine Learning Research}, 8(8).

\bibitem[{Cohen and Welling(2016)}]{cohen2016group}
Taco Cohen and Max Welling. 2016.
\newblock \href {http://proceedings.mlr.press/v48/cohenc16.html} {Group
  equivariant convolutional networks}.
\newblock In \emph{Proceedings of the 33nd International Conference on Machine
  Learning, {ICML} 2016, New York City, NY, USA, June 19-24, 2016}, volume~48
  of \emph{{JMLR} Workshop and Conference Proceedings}, pages 2990--2999.
  JMLR.org.

\bibitem[{Cotterell and Sch{\"u}tze(2018)}]{cotterell2018joint}
Ryan Cotterell and Hinrich Sch{\"u}tze. 2018.
\newblock \href {https://doi.org/10.1162/tacl_a_00003} {Joint semantic
  synthesis and morphological analysis of the derived word}.
\newblock \emph{Transactions of the Association for Computational Linguistics},
  6:33--48.

\bibitem[{Gordon et~al.(2020)Gordon, Lopez{-}Paz, Baroni, and
  Bouchacourt}]{Gordon2020Permutation}
Jonathan Gordon, David Lopez{-}Paz, Marco Baroni, and Diane Bouchacourt. 2020.
\newblock \href {https://openreview.net/forum?id=SylVNerFvr} {Permutation
  equivariant models for compositional generalization in language}.
\newblock In \emph{8th International Conference on Learning Representations,
  {ICLR} 2020, Addis Ababa, Ethiopia, April 26-30, 2020}. OpenReview.net.

\bibitem[{Hochreiter and Schmidhuber(1997)}]{hochreiter1997long}
Sepp Hochreiter and J{\"u}rgen Schmidhuber. 1997.
\newblock Long short-term memory.
\newblock \emph{Neural computation}, 9(8):1735--1780.

\bibitem[{Japkowicz et~al.(2000)}]{japkowicz2000learning}
Nathalie Japkowicz et~al. 2000.
\newblock Learning from imbalanced data sets: a comparison of various
  strategies.
\newblock In \emph{AAAI workshop on learning from imbalanced data sets},
  volume~68, pages 10--15.

\bibitem[{Jia and Liang(2016)}]{data-recombination-copy}
Robin Jia and Percy Liang. 2016.
\newblock \href {https://doi.org/10.18653/v1/P16-1002} {Data recombination for
  neural semantic parsing}.
\newblock In \emph{Proceedings of the 54th Annual Meeting of the Association
  for Computational Linguistics (Volume 1: Long Papers)}, pages 12--22, Berlin,
  Germany. Association for Computational Linguistics.

\bibitem[{Johnson et~al.(2017)Johnson, Hariharan, van~der Maaten, Fei{-}Fei,
  Zitnick, and Girshick}]{johnson2017clevr}
Justin Johnson, Bharath Hariharan, Laurens van~der Maaten, Li~Fei{-}Fei,
  C.~Lawrence Zitnick, and Ross~B. Girshick. 2017.
\newblock \href {https://doi.org/10.1109/CVPR.2017.215} {{CLEVR:} {A}
  diagnostic dataset for compositional language and elementary visual
  reasoning}.
\newblock In \emph{2017 {IEEE} Conference on Computer Vision and Pattern
  Recognition, {CVPR} 2017, Honolulu, HI, USA, July 21-26, 2017}, pages
  1988--1997. {IEEE} Computer Society.

\bibitem[{Jones et~al.(2012)Jones, Johnson, and Goldwater}]{jones2012semantic}
Bevan Jones, Mark Johnson, and Sharon Goldwater. 2012.
\newblock \href {https://aclanthology.org/P12-1051} {Semantic parsing with
  {B}ayesian tree transducers}.
\newblock In \emph{Proceedings of the 50th Annual Meeting of the Association
  for Computational Linguistics (Volume 1: Long Papers)}, pages 488--496, Jeju
  Island, Korea. Association for Computational Linguistics.

\bibitem[{Kiddon and Domingos(2015)}]{kiddon2015symmetry}
Chlo{\'e} Kiddon and Pedro Domingos. 2015.
\newblock Symmetry-based semantic parsing.
\newblock In \emph{Proceedings of the 2014 Workshop on Learning Semantics}.

\bibitem[{Kim and Linzen(2020)}]{kim2020cogs}
Najoung Kim and Tal Linzen. 2020.
\newblock \href {https://doi.org/10.18653/v1/2020.emnlp-main.731} {{COGS}: A
  compositional generalization challenge based on semantic interpretation}.
\newblock In \emph{Proceedings of the 2020 Conference on Empirical Methods in
  Natural Language Processing (EMNLP)}, pages 9087--9105, Online. Association
  for Computational Linguistics.

\bibitem[{Kim et~al.(2022)Kim, Linzen, and Smolensky}]{kim2022uncontrolled}
Najoung Kim, Tal Linzen, and Paul Smolensky. 2022.
\newblock \href {https://arxiv.org/abs/2212.10769} {Uncontrolled lexical
  exposure leads to overestimation of compositional generalization in
  pretrained models}.
\newblock \emph{ArXiv preprint}, abs/2212.10769.

\bibitem[{Kingma and Ba(2015)}]{kingma2014adam}
Diederik~P. Kingma and Jimmy Ba. 2015.
\newblock \href {http://arxiv.org/abs/1412.6980} {Adam: {A} method for
  stochastic optimization}.
\newblock In \emph{3rd International Conference on Learning Representations,
  {ICLR} 2015, San Diego, CA, USA, May 7-9, 2015, Conference Track
  Proceedings}.

\bibitem[{Koehn et~al.(2003)Koehn, Och, and Marcu}]{Koehn2003StatisticalPT}
Philipp Koehn, Franz~J. Och, and Daniel Marcu. 2003.
\newblock \href {https://aclanthology.org/N03-1017} {Statistical phrase-based
  translation}.
\newblock In \emph{Proceedings of the 2003 Human Language Technology Conference
  of the North {A}merican Chapter of the Association for Computational
  Linguistics}, pages 127--133.

\bibitem[{Kwiatkowski et~al.(2011)Kwiatkowski, Zettlemoyer, Goldwater, and
  Steedman}]{kwiatkowski2011lexical}
Tom Kwiatkowski, Luke Zettlemoyer, Sharon Goldwater, and Mark Steedman. 2011.
\newblock \href {https://aclanthology.org/D11-1140} {Lexical generalization in
  {CCG} grammar induction for semantic parsing}.
\newblock In \emph{Proceedings of the 2011 Conference on Empirical Methods in
  Natural Language Processing}, pages 1512--1523, Edinburgh, Scotland, UK.
  Association for Computational Linguistics.

\bibitem[{Lake et~al.(2019)Lake, Linzen, and Baroni}]{lake2019human}
B.~Lake, Tal Linzen, and M.~Baroni. 2019.
\newblock Human few-shot learning of compositional instructions.
\newblock In \emph{CogSci}.

\bibitem[{Lake and Baroni(2018)}]{lake2018generalization}
Brenden~M. Lake and Marco Baroni. 2018.
\newblock \href {http://proceedings.mlr.press/v80/lake18a.html} {Generalization
  without systematicity: On the compositional skills of sequence-to-sequence
  recurrent networks}.
\newblock In \emph{Proceedings of the 35th International Conference on Machine
  Learning, {ICML} 2018, Stockholmsm{\"{a}}ssan, Stockholm, Sweden, July 10-15,
  2018}, volume~80 of \emph{Proceedings of Machine Learning Research}, pages
  2879--2888. {PMLR}.

\bibitem[{Lewis and Stearns(1968)}]{lewis1968syntax}
Philip~M Lewis and Richard~Edwin Stearns. 1968.
\newblock Syntax-directed transduction.
\newblock \emph{Journal of the ACM (JACM)}, 15(3):465--488.

\bibitem[{Liu et~al.(2021{\natexlab{a}})Liu, An, Lin, Liu, Chen, Lou, Wen,
  Zheng, and Zhang}]{liu2021learning}
Chenyao Liu, Shengnan An, Zeqi Lin, Qian Liu, Bei Chen, Jian-Guang Lou, Lijie
  Wen, Nanning Zheng, and Dongmei Zhang. 2021{\natexlab{a}}.
\newblock \href {https://doi.org/10.18653/v1/2021.findings-acl.97} {Learning
  algebraic recombination for compositional generalization}.
\newblock In \emph{Findings of the Association for Computational Linguistics:
  ACL-IJCNLP 2021}, pages 1129--1144, Online. Association for Computational
  Linguistics.

\bibitem[{Liu et~al.(2021{\natexlab{b}})Liu, Kusner, and
  Blunsom}]{liu2021counterfactual}
Qi~Liu, Matt Kusner, and Phil Blunsom. 2021{\natexlab{b}}.
\newblock \href {https://doi.org/10.18653/v1/2021.naacl-main.18}
  {Counterfactual data augmentation for neural machine translation}.
\newblock In \emph{Proceedings of the 2021 Conference of the North American
  Chapter of the Association for Computational Linguistics: Human Language
  Technologies}, pages 187--197, Online. Association for Computational
  Linguistics.

\bibitem[{Long et~al.(2016)Long, Pasupat, and Liang}]{long2016simpler}
Reginald Long, Panupong Pasupat, and Percy Liang. 2016.
\newblock \href {https://doi.org/10.18653/v1/P16-1138} {Simpler
  context-dependent logical forms via model projections}.
\newblock In \emph{Proceedings of the 54th Annual Meeting of the Association
  for Computational Linguistics (Volume 1: Long Papers)}, pages 1456--1465,
  Berlin, Germany. Association for Computational Linguistics.

\bibitem[{MacCartney and Manning(2014)}]{maccartney2014natural}
Bill MacCartney and Christopher~D Manning. 2014.
\newblock Natural logic and natural language inference.
\newblock In \emph{Computing meaning}, pages 129--147. Springer.

\bibitem[{Marois et~al.(2018)Marois, Jayram, Albouy, Kornuta, Bouhadjar, and
  Ozcan}]{marois2018transfer}
Vincent Marois, TS~Jayram, Vincent Albouy, Tomasz Kornuta, Younes Bouhadjar,
  and Ahmet~S Ozcan. 2018.
\newblock \href {https://arxiv.org/abs/1811.06529} {On transfer learning using
  a mac model variant}.
\newblock \emph{ArXiv preprint}, abs/1811.06529.

\bibitem[{McCoy et~al.(2020)McCoy, Frank, and Linzen}]{mccoy2020does}
R.~Thomas McCoy, Robert Frank, and Tal Linzen. 2020.
\newblock \href {https://doi.org/10.1162/tacl_a_00304} {Does syntax need to
  grow on trees? sources of hierarchical inductive bias in sequence-to-sequence
  networks}.
\newblock \emph{Transactions of the Association for Computational Linguistics},
  8:125--140.

\bibitem[{Misra and Artzi(2016)}]{misra2016neural}
Dipendra~Kumar Misra and Yoav Artzi. 2016.
\newblock \href {https://doi.org/10.18653/v1/D16-1183} {Neural shift-reduce
  {CCG} semantic parsing}.
\newblock In \emph{Proceedings of the 2016 Conference on Empirical Methods in
  Natural Language Processing}, pages 1775--1786, Austin, Texas. Association
  for Computational Linguistics.

\bibitem[{Montague(1970{\natexlab{a}})}]{montague1970english}
Richard Montague. 1970{\natexlab{a}}.
\newblock English as a formal language. linguaggi nella societae nella tecnica.
\newblock \emph{B. Visentini (red.), Mediolan, Edizioni di Comunit{\'a}}.

\bibitem[{Montague(1970{\natexlab{b}})}]{montague1970universal}
Richard Montague. 1970{\natexlab{b}}.
\newblock Universal grammar.
\newblock \emph{Theoria}, 36(3):373--398.

\bibitem[{Narasimhan et~al.(2015)Narasimhan, Barzilay, and
  Jaakkola}]{narasimhan2015unsupervised}
Karthik Narasimhan, Regina Barzilay, and Tommi Jaakkola. 2015.
\newblock \href {https://doi.org/10.1162/tacl_a_00130} {An unsupervised method
  for uncovering morphological chains}.
\newblock \emph{Transactions of the Association for Computational Linguistics},
  3:157--167.

\bibitem[{Noether(1918)}]{Noether1918}
E.~Noether. 1918.
\newblock \href {http://eudml.org/doc/59024} {Invariante variationsprobleme}.
\newblock \emph{Nachrichten von der Gesellschaft der Wissenschaften zu
  Göttingen, Mathematisch-Physikalische Klasse}, 1918:235--257.

\bibitem[{Perez et~al.(2018)Perez, Strub, de~Vries, Dumoulin, and
  Courville}]{perez2018film}
Ethan Perez, Florian Strub, Harm de~Vries, Vincent Dumoulin, and Aaron~C.
  Courville. 2018.
\newblock \href
  {https://www.aaai.org/ocs/index.php/AAAI/AAAI18/paper/view/16528} {Film:
  Visual reasoning with a general conditioning layer}.
\newblock In \emph{Proceedings of the Thirty-Second {AAAI} Conference on
  Artificial Intelligence, (AAAI-18), the 30th innovative Applications of
  Artificial Intelligence (IAAI-18), and the 8th {AAAI} Symposium on
  Educational Advances in Artificial Intelligence (EAAI-18), New Orleans,
  Louisiana, USA, February 2-7, 2018}, pages 3942--3951. {AAAI} Press.

\bibitem[{Pham et~al.(2018)Pham, Niehues, and Waibel}]{pham2018towards}
Ngoc-Quan Pham, Jan Niehues, and Alexander Waibel. 2018.
\newblock \href {https://doi.org/10.18653/v1/W18-2712} {Towards one-shot
  learning for rare-word translation with external experts}.
\newblock In \emph{Proceedings of the 2nd Workshop on Neural Machine
  Translation and Generation}, pages 100--109, Melbourne, Australia.
  Association for Computational Linguistics.

\bibitem[{Popel and Bojar(2018)}]{popel2018training}
Martin Popel and Ond{\v{r}}ej Bojar. 2018.
\newblock \href {https://arxiv.org/abs/1804.00247} {Training tips for the
  transformer model}.
\newblock \emph{ArXiv preprint}, abs/1804.00247.

\bibitem[{Prabhu and Kann(2020)}]{Prabhu2020MakingAP}
Nikhil Prabhu and Katharina Kann. 2020.
\newblock \href {https://aclanthology.org/2020.aacl-srw.13} {Making a point:
  Pointer-generator transformers for disjoint vocabularies}.
\newblock In \emph{Proceedings of the 1st Conference of the Asia-Pacific
  Chapter of the Association for Computational Linguistics and the 10th
  International Joint Conference on Natural Language Processing: Student
  Research Workshop}, pages 85--92, Suzhou, China. Association for
  Computational Linguistics.

\bibitem[{Qiu et~al.(2022)Qiu, Shaw, Pasupat, Nowak, Linzen, Sha, and
  Toutanova}]{qiu2021improving}
Linlu Qiu, Peter Shaw, Panupong Pasupat, Pawel Nowak, Tal Linzen, Fei Sha, and
  Kristina Toutanova. 2022.
\newblock \href {https://doi.org/10.18653/v1/2022.naacl-main.323} {Improving
  compositional generalization with latent structure and data augmentation}.
\newblock In \emph{Proceedings of the 2022 Conference of the North American
  Chapter of the Association for Computational Linguistics: Human Language
  Technologies}, pages 4341--4362, Seattle, United States. Association for
  Computational Linguistics.

\bibitem[{Raffel et~al.(2020)Raffel, Shazeer, Roberts, Lee, Narang, Matena,
  Zhou, Li, and Liu}]{raffel2019exploring}
Colin Raffel, Noam Shazeer, Adam Roberts, Katherine Lee, Sharan Narang, Michael
  Matena, Yanqi Zhou, Wei Li, and Peter~J. Liu. 2020.
\newblock \href {http://jmlr.org/papers/v21/20-074.html} {Exploring the limits
  of transfer learning with a unified text-to-text transformer}.
\newblock \emph{J. Mach. Learn. Res.}, 21:140:1--140:67.

\bibitem[{Ramesh et~al.(2021)Ramesh, Pavlov, Goh, Gray, Voss, Radford, Chen,
  and Sutskever}]{ramesh2021zero}
Aditya Ramesh, Mikhail Pavlov, Gabriel Goh, Scott Gray, Chelsea Voss, Alec
  Radford, Mark Chen, and Ilya Sutskever. 2021.
\newblock \href {http://proceedings.mlr.press/v139/ramesh21a.html} {Zero-shot
  text-to-image generation}.
\newblock In \emph{Proceedings of the 38th International Conference on Machine
  Learning, {ICML} 2021, 18-24 July 2021, Virtual Event}, volume 139 of
  \emph{Proceedings of Machine Learning Research}, pages 8821--8831. {PMLR}.

\bibitem[{Resnick et~al.(2019)Resnick, Gupta, Foerster, Dai, and
  Cho}]{resnick2019capacity}
Cinjon Resnick, Abhinav Gupta, Jakob Foerster, Andrew~M Dai, and Kyunghyun Cho.
  2019.
\newblock \href {https://arxiv.org/abs/1910.11424} {Capacity, bandwidth, and
  compositionality in emergent language learning}.
\newblock \emph{ArXiv preprint}, abs/1910.11424.

\bibitem[{Shaw et~al.(2021)Shaw, Chang, Pasupat, and
  Toutanova}]{shaw2020compositional}
Peter Shaw, Ming-Wei Chang, Panupong Pasupat, and Kristina Toutanova. 2021.
\newblock \href {https://doi.org/10.18653/v1/2021.acl-long.75} {Compositional
  generalization and natural language variation: Can a semantic parsing
  approach handle both?}
\newblock In \emph{Proceedings of the 59th Annual Meeting of the Association
  for Computational Linguistics and the 11th International Joint Conference on
  Natural Language Processing (Volume 1: Long Papers)}, pages 922--938, Online.
  Association for Computational Linguistics.

\bibitem[{Shorten and Khoshgoftaar(2019)}]{shorten2019survey}
Connor Shorten and Taghi~M Khoshgoftaar. 2019.
\newblock A survey on image data augmentation for deep learning.
\newblock \emph{Journal of Big Data}, 6(1):1--48.

\bibitem[{Simeonov et~al.(2022)Simeonov, Du, Yen-Chen, Rodriguez, Kaelbling,
  Lozano-Perez, and Agrawal}]{simeonov2022se}
Anthony Simeonov, Yilun Du, Lin Yen-Chen, Alberto Rodriguez, Leslie~Pack
  Kaelbling, Tomas Lozano-Perez, and Pulkit Agrawal. 2022.
\newblock \href {https://arxiv.org/abs/2211.09786} {Se (3)-equivariant
  relational rearrangement with neural descriptor fields}.
\newblock \emph{ArXiv preprint}, abs/2211.09786.

\bibitem[{van~den Oord et~al.(2017)van~den Oord, Vinyals, and
  Kavukcuoglu}]{oord2017neural}
A{\"{a}}ron van~den Oord, Oriol Vinyals, and Koray Kavukcuoglu. 2017.
\newblock \href
  {https://proceedings.neurips.cc/paper/2017/hash/7a98af17e63a0ac09ce2e96d03992fbc-Abstract.html}
  {Neural discrete representation learning}.
\newblock In \emph{Advances in Neural Information Processing Systems 30: Annual
  Conference on Neural Information Processing Systems 2017, December 4-9, 2017,
  Long Beach, CA, {USA}}, pages 6306--6315.

\bibitem[{Vaswani et~al.(2017)Vaswani, Shazeer, Parmar, Uszkoreit, Jones,
  Gomez, Kaiser, and Polosukhin}]{vaswani2017attention}
Ashish Vaswani, Noam Shazeer, Niki Parmar, Jakob Uszkoreit, Llion Jones,
  Aidan~N. Gomez, Lukasz Kaiser, and Illia Polosukhin. 2017.
\newblock \href
  {https://proceedings.neurips.cc/paper/2017/hash/3f5ee243547dee91fbd053c1c4a845aa-Abstract.html}
  {Attention is all you need}.
\newblock In \emph{Advances in Neural Information Processing Systems 30: Annual
  Conference on Neural Information Processing Systems 2017, December 4-9, 2017,
  Long Beach, CA, {USA}}, pages 5998--6008.

\bibitem[{Wang et~al.(2018)Wang, Pham, Dai, and Neubig}]{wang2018switchout}
Xinyi Wang, Hieu Pham, Zihang Dai, and Graham Neubig. 2018.
\newblock \href {https://doi.org/10.18653/v1/D18-1100} {{S}witch{O}ut: an
  efficient data augmentation algorithm for neural machine translation}.
\newblock In \emph{Proceedings of the 2018 Conference on Empirical Methods in
  Natural Language Processing}, pages 856--861, Brussels, Belgium. Association
  for Computational Linguistics.

\bibitem[{Wei and Zou(2019)}]{wei2019eda}
Jason Wei and Kai Zou. 2019.
\newblock \href {https://doi.org/10.18653/v1/D19-1670} {{EDA}: Easy data
  augmentation techniques for boosting performance on text classification
  tasks}.
\newblock In \emph{Proceedings of the 2019 Conference on Empirical Methods in
  Natural Language Processing and the 9th International Joint Conference on
  Natural Language Processing (EMNLP-IJCNLP)}, pages 6382--6388, Hong Kong,
  China. Association for Computational Linguistics.

\bibitem[{White and Cotterell(2022)}]{white-cotterell-2022-equivariant}
Jennifer~C. White and Ryan Cotterell. 2022.
\newblock \href {https://aclanthology.org/2022.coling-1.412} {Equivariant
  transduction through invariant alignment}.
\newblock In \emph{Proceedings of the 29th International Conference on
  Computational Linguistics}, pages 4651--4663, Gyeongju, Republic of Korea.
  International Committee on Computational Linguistics.

\bibitem[{Wolf et~al.(2019)Wolf, Debut, Sanh, Chaumond, Delangue, Moi, Cistac,
  Rault, Louf, Funtowicz et~al.}]{wolf2019huggingface}
Thomas Wolf, Lysandre Debut, Victor Sanh, Julien Chaumond, Clement Delangue,
  Anthony Moi, Pierric Cistac, Tim Rault, R{\'e}mi Louf, Morgan Funtowicz,
  et~al. 2019.
\newblock \href {https://arxiv.org/abs/1910.03771} {Huggingface's transformers:
  State-of-the-art natural language processing}.
\newblock \emph{ArXiv preprint}, abs/1910.03771.

\bibitem[{Wu et~al.(2016)Wu, Schuster, Chen, Le, Norouzi, Macherey, Krikun,
  Cao, Gao, Macherey, Klingner, Shah, Johnson, Liu, Łukasz Kaiser, Gouws,
  Kato, Kudo, Kazawa, Stevens, Kurian, Patil, Wang, Young, Smith, Riesa,
  Rudnick, Vinyals, Corrado, Hughes, and Dean}]{googlemt}
Yonghui Wu, Mike Schuster, Zhifeng Chen, Quoc~V. Le, Mohammad Norouzi, Wolfgang
  Macherey, Maxim Krikun, Yuan Cao, Qin Gao, Klaus Macherey, Jeff Klingner,
  Apurva Shah, Melvin Johnson, Xiaobing Liu, Łukasz Kaiser, Stephan Gouws,
  Yoshikiyo Kato, Taku Kudo, Hideto Kazawa, Keith Stevens, George Kurian,
  Nishant Patil, Wei Wang, Cliff Young, Jason Smith, Jason Riesa, Alex Rudnick,
  Oriol Vinyals, Greg Corrado, Macduff Hughes, and Jeffrey Dean. 2016.
\newblock \href {https://arxiv.org/abs/1609.08144} {Google's neural machine
  translation system: Bridging the gap between human and machine translation}.
\newblock \emph{ArXiv preprint}, abs/1609.08144.

\bibitem[{Yi et~al.(2018)Yi, Wu, Gan, Torralba, Kohli, and
  Tenenbaum}]{yi2018neural}
Kexin Yi, Jiajun Wu, Chuang Gan, Antonio Torralba, Pushmeet Kohli, and Josh
  Tenenbaum. 2018.
\newblock \href
  {https://proceedings.neurips.cc/paper/2018/hash/5e388103a391daabe3de1d76a6739ccd-Abstract.html}
  {Neural-symbolic {VQA:} disentangling reasoning from vision and language
  understanding}.
\newblock In \emph{Advances in Neural Information Processing Systems 31: Annual
  Conference on Neural Information Processing Systems 2018, NeurIPS 2018,
  December 3-8, 2018, Montr{\'{e}}al, Canada}, pages 1039--1050.

\bibitem[{Zhang et~al.(2022)Zhang, Yang, and Yang}]{zhang2022treemix}
Le~Zhang, Zichao Yang, and Diyi Yang. 2022.
\newblock \href {https://doi.org/10.18653/v1/2022.naacl-main.385} {{T}ree{M}ix:
  Compositional constituency-based data augmentation for natural language
  understanding}.
\newblock In \emph{Proceedings of the 2022 Conference of the North American
  Chapter of the Association for Computational Linguistics: Human Language
  Technologies}, pages 5243--5258, Seattle, United States. Association for
  Computational Linguistics.

\end{thebibliography}
\bibliographystyle{acl_natbib}
\newpage
\appendix \label{sec:appendix}
\section{Proof of \cref{prop:fswap}}
\label{app:fswap}
\begin{proof}
The lexicons that we learn only unary type relations and a semantic correspondence relation 
$\lexicon = (\{r_{{\typefn_k}}\}, r_{\equivfn})$. As noted there, we make the following additional assumptions (satisfied by our lexicon learning algorithms):
\begin{enumerate}[label=(\roman*)]
    \item \emph{Types are disjoint}, i.e. every symbol belongs to a single type: $\forall_x \in \Sigma, |\typefn_{x}| = |\{r_{\typefn_k} \mid r_{\typefn_k}(x) =1\}| = 1 $.
    \item \emph{Semantic correspondences are one-to-many from \textit{text} to \textit{meaning}}. This means that no two text symbols can translate into the same meaning symbol: $E_i \cap E_j = \mathbb{1}_{x_i = x_j}$ and all $r_{\equivfn}(x \notin \inp_\text{text}, y) = r_{\equivfn}(y, x \notin \inp_\text{meaning}) = 0$. %
    \item \emph{Semantic correspondence is type preserving}: all symbols in a correspondence class have the same type  $\typefn_{e_i \in E_i} = \{r_{\typefn_{E_i}}\}$.
\end{enumerate}

To show that $f$ is an $\lexicon$-homomorphism, we want to show that $r_\epsilon(f(x_1), f(x_2))=r_\epsilon(x_1, x_2)$ for any $x_1, x_2$. The transformation function and all the definitions are symmetric to indices $i$ and $j$ ($i-j$ symmetry), so it is sufficient to show the correspondence relations stay the same for below cases only:
\begin{enumerate}[label=(\alph*)]
\item $x_1=x_{i}, x_2=x_{i}$: \begin{align*}
r_{\equivfn}(f(x_i), f(x_i)) = r_{\equivfn}(x_j, x_j) = 0 = r_{\equivfn}(x_i, x_i) 
\end{align*}
\textit{(by ii)}
\item $x_1=x_{i}, x_2=x_{j}$: \begin{align*}
r_{\equivfn}(f(x_i), f(x_j)) = r_{\equivfn}(x_j, x_i) = 0 =r_{\equivfn}(x_i, x_j) 
\end{align*}
\textit{(by ii)}
\item $x_1=x_{i}, x_2 \in E_i$:
\begin{align*}
r_{\equivfn}(f(x_i), f(x_2)) &= r_{\equivfn}(x_j, x' \in  E_j) \\&= 1 = r_{\equivfn}(x_i, x_2)  
\end{align*}
\textit{(by definition of $E_i$ and $E_j$)}
\item $x_1=x_{i}, x_2 \in E_j$:
\begin{align*}
r_{\equivfn}(f(x_i), f(x_2)) &= r_{\equivfn}(x_j, x' \in  E_i) \\&= \mathbb{1}_{x_i=x_j} = r_{\equivfn}(x_i, x_2)
\end{align*}
\textit{(by ii)}
\item $x_1=x_{i}, x_2 \notin \{\{x_i\} \cup \{x_j\} \cup E_i, E_j \}$:
\begin{align*}
r_{\equivfn}(f(x_i), f(x_2)) &= r_{\equivfn}(x_j, x_2) \\&= 0 = r_{\equivfn}(x_i, x_2)  
\end{align*}
\item $x_1=x_{i}, x_2 \notin \{\{x_i\} \cup \{x_j\} \cup E_i, E_j \}$: \\[.5em]
\textit{same steps as (e)}
\item $x_1 \in E_i, x_2=x_i$:
\begin{align*}
r_{\equivfn}(f(x_1), f(x_i)) &= r_{\equivfn}(x' \in E_j, x_j) \\&= 0 = r_{\equivfn}(x_1, x_i)  
\end{align*}
\textit{(by ii)}
\item $x_1 \in E_i, x_2=x_j$: \textit{same steps as (g)}
\item $x_1 \in E_i, x_2 \in \{\{x_i\} \cup \{x_j\} \cup E_i, E_j \}$:
\begin{align*}
r_{\equivfn}(f(x_1), f(x_2)) &= r_{\equivfn}(x' \in E_j, x_2) \\&=  0 = r_{\equivfn}(x_1, x_2)
\end{align*}
\textit{(by ii)}
\end{enumerate}

Finally, we require $r_{\typefn}(x) = r_{\typefn}(f(x))$ for any $x$ and $\typefn$. Since we assume all items in $E_i$ belong to a type matching $x_i$ (likewise for $j$), and types are disjoint, this follows immediately from the definition of $f$, which only swaps symbols of the same type.
\end{proof}

\section{Enumerating $\lexicon$-homomorphisms}
\label{app:alg}

A simple algorithm is given below:
\begin{algorithm}
\caption{$\lexicon$-homomorphism enumeration}
\label{alg:lhom}
\begin{algorithmic}
\footnotesize
\State{\textbf{input}: Lexicon $\lexicon = (\vocab, r_1, \ldots, r_n)$}
\For{$f \in \vocab^{\vocab}$}
    \State $\textit{h} \gets 1$
    \For {$i = 1..n, \inpi_a..\inpi_b \in \vocab^p$}
        \If{$r(\inpi_a, \ldots, \inpi_b) \neq r(f(\inpi_a), \ldots, f(\inpi_b))$}
        \State $\textit{h} \gets 0$
        \EndIf
    \EndFor
    \If{$h$}
        \State{\textbf{yield} $f$}
    \EndIf
\EndFor
\end{algorithmic}
\end{algorithm}

\section{Implementation Details}\label{app:impdetails}
\begin{figure*}[t]
    \centering
    \includegraphics[width=\textwidth]{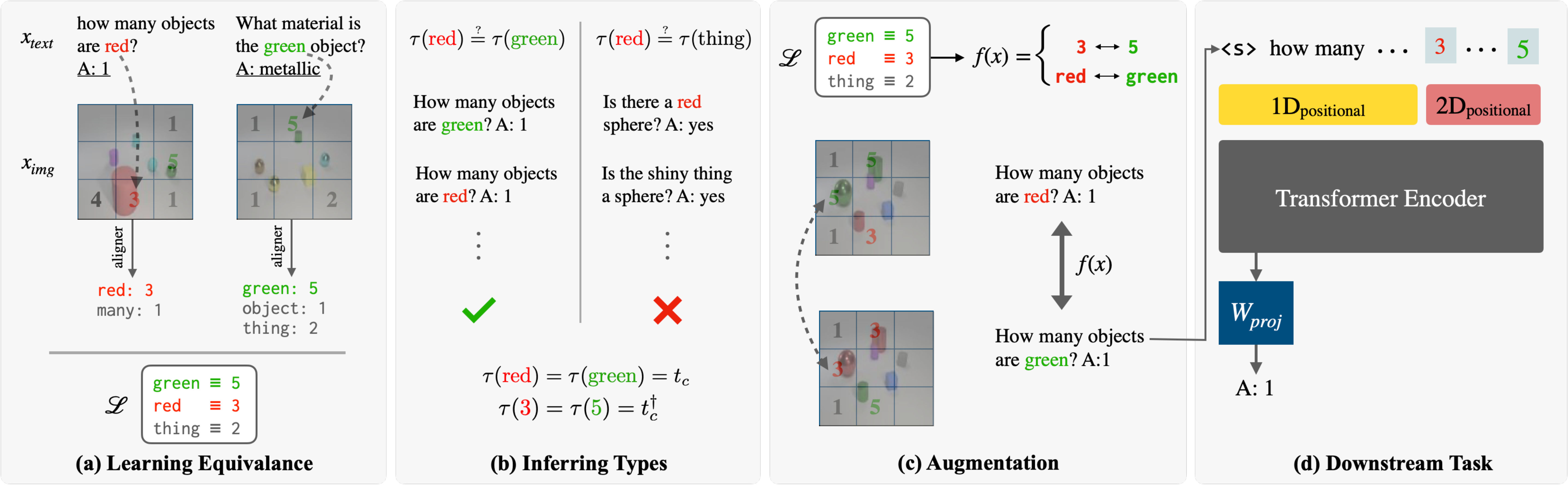}
    \caption{Overview of our approach in VQA. We discretize images using a VQVAE \cite{oord2017neural} learned from the training data. This discretization represents every image as a sequence of categorical codes. (a) We run a statistical aligner on ($\mathbf{x}_{\textrm{text}}$, $\mathbf{x}_{\textrm{img}}$) pairs to find word--visual token alignments within individual examples, then use these alignments to construct a global lexicon. (b) Each entry in the lexicon is assigned a type based on the context in which it occurs. (c) Next, we find \emph{homomorphisms} of this lexicon, and use these as data augmentation functions to generate new training examples. (d) Finally, we train a neural sequence model on the augmented dataset.}
    \label{fig:method}
\end{figure*}
\subsection{VQVAE Details}\label{app:vqvae}
We use a discrete variational auto-encoder \cite{oord2017neural} to encode the images $16 \times 16$ grids of discrete codes. We used a code-book with $n=32$ tokens associated with $d=64$ dimensional learned latent vectors. The original image size $(480,320)$ is cropped to $(440,300)$ and resize our images into $(128, 128)$ pixels. The encoder convolutional neural network has three down-sampling layers which output $16 \times 16 \times d$ size hidden representations. For encoder and decoder CNN architectures, we follow the implementation provided in a public Pytorch implementation\footnote{\url{https://github.com/ritheshkumar95/pytorch-vqvae}} by adding one more up-sampling and down-sampling layer to adjust our image size.

We use exponential moving average to update latent vectors as in official implementation\footnote{\url{https://github.com/deepmind/sonnet/blob/v2/sonnet/src/nets/vqvae.py}}
We train the model on the images of the same training data and did not use any external data.

We use batch size of $512$, and learning rate $0.0003$ with the Adam optimizer \cite{kingma2014adam}. We clip the gradients to $5.0$. Hyperparameters were selected by sweeping $d$ over $\{64,128\}$, image sizes over $\{128,144\}$, and $n$ over $\{24, 32, 48\}$ to maximize the the number of aligned tokens in the lexicon. For each experiments in \cref{tab:clevr_cogent}, we run VQVAE for 4 random seeds and select the codebook that gives the largest IBM model likelihood for training data. Each experiment takes 10 hours in 4 NVIDIA V100 GPUs.

\subsection{VQA Transformer Details}\label{app:vqatransformer}
The Transformer takes tokenized images $\vx_{I}$ and the question $\vx_{Q}$ and outputs answers as follows:
\begin{align}\label{eq:vqatransformer}
\begin{split}
    c_{\mathbf{x}_{I}}  &= \operatorname{VQVAE_{\textrm{enc}}}({\mathbf{x}_{I}}) \\
    e_Q &= W_Q  \mathbf{x}_Q + \operatorname{1D_{\textrm{positional}}}(\mathbf{x}_Q) \\ 
    e_{\mathbf{x}_{I}} &= W_c c_{\mathbf{x}_{I}}  + \operatorname{2D_{\textrm{positional}}}(c_{\mathbf{x}}) \\
    h &= \operatorname{Transformer}([e_Q \,e_{\mathbf{x}_I}]) \\
    \vx_{A} &= \operatorname{argmax} \textrm{softmax}(W_{\textrm{proj}} h_{\textrm{start}})
    \end{split}
\end{align}
We follow the hyper-paramters provided in \cite{popel2018training}. Transformers have 4 heads, $512$-dimensional hidden vectors (same with embedding sizes) and $10$ layers. We provide the dimensions in \cref{eq:vqatransformer}:
\begin{align}
\begin{split}
    \mathbf{x}_{I} &: 3 \times 128 \times 128\\
    c_{\mathbf{x}_{I}} &: 32 \times 16 \times 16\\
    W_c &: 512 \times 32 \\
    e_{\mathbf{x}_{I}} &: 512 \times (16 \times 16) \\
    e_Q &: 512 \times |V_{text}| \\
    W_Q &: 512 \times |\mathcal{V}_{\textrm{text}}| \\
    h &: 512 \times (|Q| + 16 \times 16) \\
    h_{\textrm{start}} &: 512 \times 1 \\
    W_{\textrm{proj}} &: 512 \times |\mathcal{V}_{\textrm{text}}| \\
    \end{split}
\end{align}
Models are trained using the Adam optimizer with and Noam learning rate scheduler \cite{vaswani2017attention} with $lr=1.0$ and $16k$ warming steps as provided in \citet{popel2018training}. We use a batch size of $1024$ and we train for $200k$ steps, which takes 48 hours on 8 NVIDIA V100 GPUs. In \cref{fig:method}, we provide the sketch of overall pipeline.

\subsection{Baselines: LSTM Details}\label{app:lstm}
We use the implementation provided by \cite{akyurek2021lexicon}, increasing the number of training iterations from $8k$ to $15k$ for augmented training runs in \COGS, \SCAN datasets. For the \ALCH dataset, we optimize iteration count over $\{8k, 15k, 25k, 50k\}$ based on validation accuracy, and found $25k$ to be optimal. For the \CLEVR dataset, we optimize itreation count over $\{8k, 15k, 25k, 50k\}$ for \CLEVR and \CLEVR-\CoGenT dataset based on \CLEVR's validation accuracy.

\subsection{Baselines: T5 Details}
We use the Huggingface \cite{wolf2019huggingface} implementation T5-base model. The difference between our T5 baselines results and the results in \citet{qiu2021improving} due to their usage of different intermediate representation for the output in order to keep our evaluation consistent with other previous work. We try to optimize (learning rate, learning rate scheduler) and training parameters (iteration count) of \citet{qiu2021improving} and \cite{akyurek2021lexicon}, use the best setting for the given dataset.

\subsection{Alignment Model Details}\label{app:ibm}
In our experiments, we use the best alignment method reported in \cite{akyurek2021lexicon}, which is IBM Model 2 for all datasets except the SCAN dataset that uses their proposed algorithm, to obtain our initial alignments $\mathcal{A} = \{(x_i, x_j)$: set of tuples contains aligned tokens. We run alignment algorithms between  $\inp_\text{text}$ and $\inp_\text{meaning}$. For \SCAN and \COGS, $\inp_\text{text}$ is the actual inputs, $\inp_\text{meaning}$ is the actual outputs. In \ALCH, $\inp_\text{text}$ is instructions, $\inp_\text{meaning}$ is beaker states. In VQA experiments, $\inp_\text{text}$ question and answer words, $\inp_\text{meaning}$ VQVAE codes. We disable \emph{diagonalization} in FastAlign as it includes non-language structured VQVAE codes. %
\section{Lexicons}\label{app:lexlearn}

\subsection{Lexicon Learning}
\paragraph{Extracting semantic correspondences $\boldsymbol{r_{\equivfn}(\inpi_i, \inpi_j)}$} 
Given the initial alignments $\mathcal{A}$ in \cref{app:ibm}, we remove every $\inpi_j$ that is not aligned to at least 1\% of occurrences of $\inpi_i$ in the dataset. We then produce a \emph{one-to-many} lexicon by deleting lexicon entries $(\inpi_i, \inpi_j)$ and $(\inpi_i', \inpi_j)$ when both exist. With, these alignment creates entries in $r_{\equivfn}(\inpi_i, \inpi_j) = \mathbb{1}_{(\inpi_i, \inpi_j) \in \mathcal{A}}$

\paragraph{Extracting Types $\boldsymbol{\mathbf{r_{\typefn}(\mathrm{\inpi})}}$} 
Given the partition of the data points ($\inp_\text{text}, \inp_\text{meaning}$), our type finding algorithm is essentially \emph{unsupervised clustering} of the text symbols in $\inp_\text{text}$. The types of matching $\inp_\text{meaning}$ symbols are automatically determined by the correspondence relation, $r_{\equivfn}$ found above. In all our datasets $\inp_\text{text}$ is English, so the symbols that goes into following clustering algorithm are actual words. 

Following  \citet{clark2007polynomial} and \citet{geca}, we assign types to individual words based on their environments. For each symbol, $\inpi \in \vocab$, that has at least one equivalent symbol in $\mathcal{A}$,  we define the context $\kappa(\inpi) = \{ (\alpha, \beta) : \alpha \inpi \beta \in \goodinpspace \}$: the set of strings $(\alpha, \beta)$ that appear surrounding $\inpi$ in the training set. (If the two examples in \cref{fig:teaser} formed the entire training set, we would have $\kappa(\textit{yellow}) = \kappa(\textit{green}) = \{ (\textit{Q: How many}, ~\textit{objects? A: 1}) \}$.). \footnote{The environment's window size $w=|\alpha|=|\beta|$ is a fixed hyper-parameter similar to \citet{geca}. We optimize it over $w \in \{1,2,5,10,15\}$ in \ALCH dataset (used $w=10$) by using the validation set, in \CLEVR and \CLEVR-\CoGenT based on \CLEVR's validation set (used $w=10$). For \COGS and \SCAN datasets, we resort $w=1$ to enable learning of extremely rare items.}
We then represent $\vocab$ as a graph with an edge between each $\inpi_i$ and $\inpi_j$ where $\kappa(\inpi_i) \cap \kappa(\inpi_j) \neq \emptyset$ (Clark and Eyraud's \emph{syntactic congruence} relation) and $\inpi_i$ and $\inpi_j$ has same part-of-speech tag according to spaCy pipeline with \texttt{en-core-web-lm} language model \footnote{\url{https://github.com/explosion/spacy-models/releases/tag/en_core_web_lg-3.5.0}}. We assign each connected component of this graph a distinct type. This is only one possible approach to typing;  alternatives might use clustering of distributed representations.
\subsection{Extracted Lexicons}
In this section, we present lexicon entries for symbols that we learned through our typing algorithm. 
\paragraph{\SCAN}
We present equivalance relations that we extracted from \SCAN training dataset.
\begin{tabular}{lll}
\toprule
     \textbf{Source symbol} &  \textbf{Type} & \textbf{Target Symbol(s)}  \\
\midrule
     jump & $t_1$ & \texttt{I\_JUMP} \\ 
     walk & $t_1$ & \texttt{I\_WALK}\\
     run & $t_1$ & \texttt{I\_RUN}\\ 
     look & $t_1$ & \texttt{I\_LOOK} \\
     left & $t_2$ & \texttt{I\_LEFT} \\
     right & $t_2$ & \texttt{I\_RIGHT} \\
\bottomrule
\end{tabular}
\paragraph{\COGS} Since the extracted lexicon is large for semantic parsing, we present only some of the equivalance relations that we extracted from \COGS training data for reference.
\newline
\newline
\begin{tabular}{lll}
\toprule
     \textbf{Source symbol} &  \textbf{Type} & \textbf{Target Symbol(s)}  \\
\midrule
     baked & $t_1$ & \texttt{bake} \\ 
     noticed & $t_1$ & \texttt{notice}\\
     helped & $t_1$ & \texttt{help}\\ 
     dog & $t_2$ & \texttt{dog} \\
     boy & $t_2$ & \texttt{boy} \\
     sailor & $t_2$ & \texttt{sailor} \\
\bottomrule
\end{tabular}
\newline
\newline
\paragraph{\CoGenT}
We present equivalance relations that we extracted \CLEVR-\CoGenT training data. The lexicon we found includes all the color symbols. The target symbols given here are learned VQVAE codes. In \cref{tab:samples}, we show these codes on top of the images to qualitatively verify the alignments.
\newline
\newline
\begin{tabular}{lll}
\toprule
     \textbf{Source Symbol} & \textbf{Type} & \textbf{Target Symbols}  \\
\midrule
     red & $t_1$ & 9 \\ 
     purple & $t_1$ &  25, 29 \\ 
     cyan & $t_1$ &  28 \\
     blue & $t_1$ &   20 \\
     green & $t_1$ &  11 \\
     yellow & $t_1$ & 23, 18 \\
     gray & $t_1$ &  6 \\
     brown & $t_1$ &  2 \\
\bottomrule
\end{tabular}
\newline
\newline

\section{Samples \& Statistics}\label{tab:samples}
\begin{table*}[t!]
\begin{subtable}[]{\textwidth}
\label{tab:cogs}
\centering
\resizebox{\linewidth}{!}{%
\begin{tabular}{l|l|l|l}
\toprule
\multicolumn{1}{c}{\textbf{Generated Sentence}} & \multicolumn{1}{c}{\textbf{Generated Logical form}} & \multicolumn{1}{c}{\textbf{Original Sentence}} & \multicolumn{1}{c}{\textbf{Original Example Logical Form}} \\
\midrule
\textit{A cake was \underline{baked} by Scarlett .} & \texttt{\begin{tabular}[c]{@{}l@{}}cake($x_{1}$) AND bake.theme($x_{3}$,  $x_{1}$) AND \\ bake.agent ($x_{3}$, Scarlett )\end{tabular}} & \textit{A cake was \underline{stabbed} by Scarlett .} & \texttt{\begin{tabular}[c]{@{}l@{}}cake($x_{1}$) AND stab.theme ($x_{3}$, $x_{1}$) AND \\ stab.agent ($x_{3}$,  Scarlett )\end{tabular}} \\ \hline
\textit{The \underline{bunny} needed to cook .} & \texttt{\begin{tabular}[c]{@{}l@{}}*bunny($x_{1}$); need.agent($x_{2}$, $x_{1}$) AND \\  need.xcomp ($x_{2}$, $x_{4}$)  AND cook.agent($x_{4}$,  $x_{1}$)\end{tabular}} & \textit{The \underline{girl} needed to cook .} & \texttt{\begin{tabular}[c|]{@{}l@{}}*girl ($x_{1}$); need.agent ($x_{2}$, $x_{1}$) AND  \\ need.xcomp($x_{2}$, $x_{4}$) AND cook.agent ($x_{4}$, $x_{1}$)\end{tabular}} \\ \hline
\textit{The \underline{bun} hunted Emma .} & \texttt{\begin{tabular}[c]{@{}l@{}}*bun($x_{1}$); hunt.agent($x_{2}$, $x_{1}$) AND \\ hunt.theme ($x_{2}$, Emma)\end{tabular}} & \textit{The \underline{teacher} hunted Emma .} & \texttt{\begin{tabular}[c]{@{}l@{}}*teacher($x_{1}$); hunt.agent($x_{2}$, $x_{1}$) AND  \\ hunt.theme($x_{2}$, Emma)\end{tabular}} \\
\bottomrule
\end{tabular}
}
\end{subtable}
 
\begin{subtable}[]{\textwidth}
\label{tab:vqa}
\centering
\resizebox{\linewidth}{!}{%
\begin{tabular}{llll}
\multicolumn{1}{c}{\textbf{Generated Text}} & \multicolumn{1}{c}{\textbf{Generated Image}} & \multicolumn{1}{c}{\textbf{Original Text}} & \multicolumn{1}{c}{\textbf{Original Image}} \\
\midrule
 \raisebox{15mm}{\textit{\begin{tabular}[l]{@{}l@{}}How many metallic objects are  \\ either tiny \underline{yellow} things or blocks? \\  A: 1\end{tabular}}}  &  \multicolumn{1}{c|}{\includegraphics[width=0.2\textwidth]{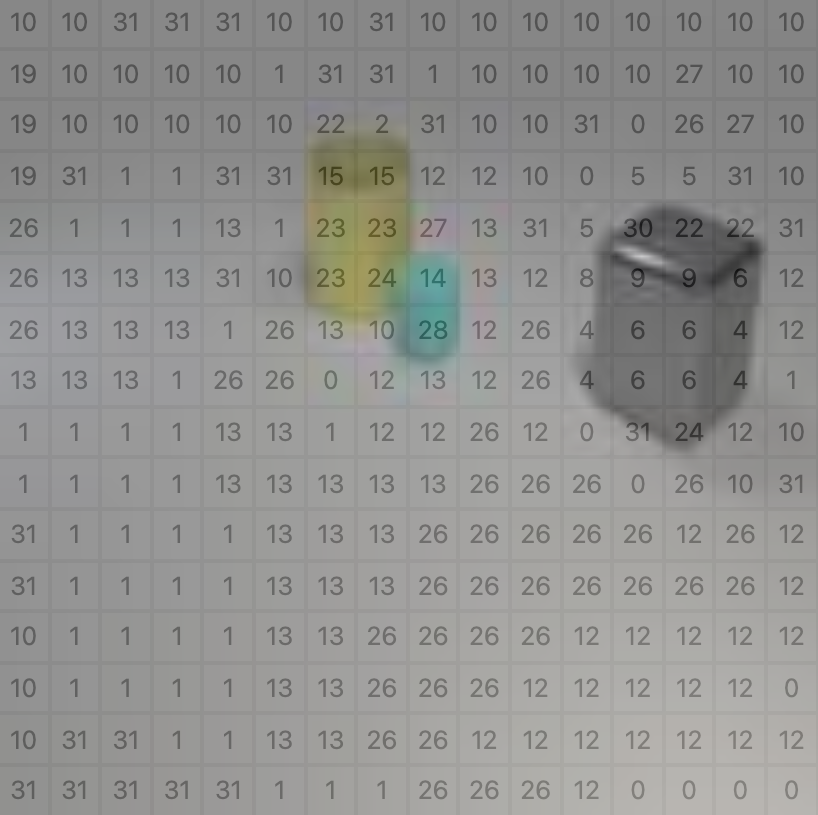}} & 
 \raisebox{15mm}{\textit{\begin{tabular}[l]{@{}l@{}}How many metallic objects are  \\ either tiny \underline{red} things or blocks? \\  A: 1\end{tabular}}}  &   \multicolumn{1}{c}{\includegraphics[width=0.2\textwidth]{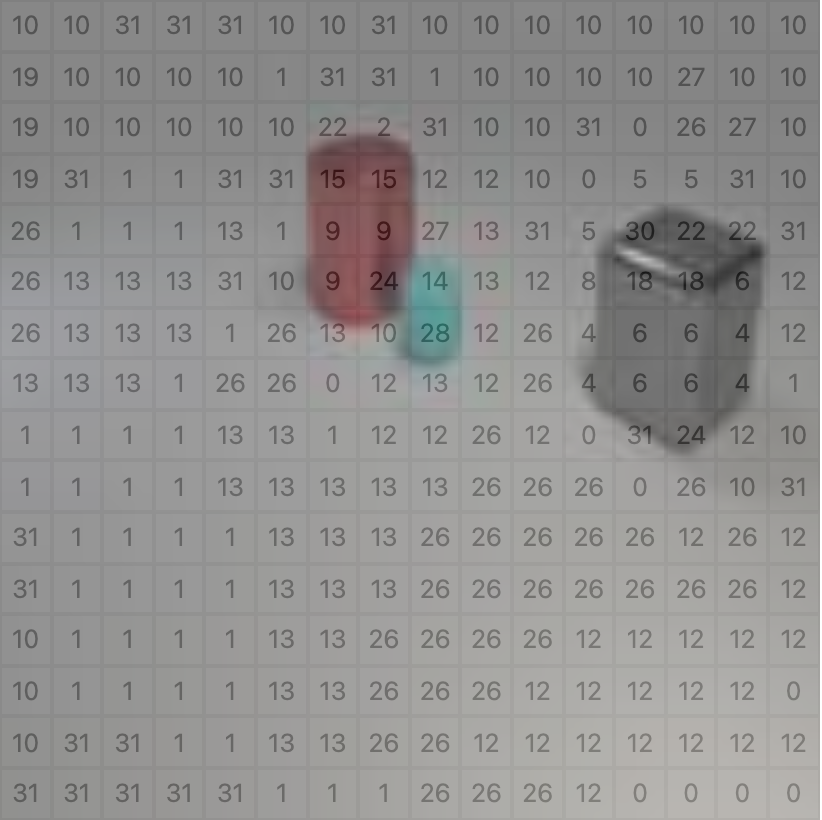}} \\
 \raisebox{15mm}{\textit{\begin{tabular}[l]{@{}l@{}}
What is the size of the other object that is \\ the same material as the big \underline{brown} thing \\ A: Large \end{tabular}}}  &  \multicolumn{1}{c|}{\includegraphics[width=0.2\textwidth]{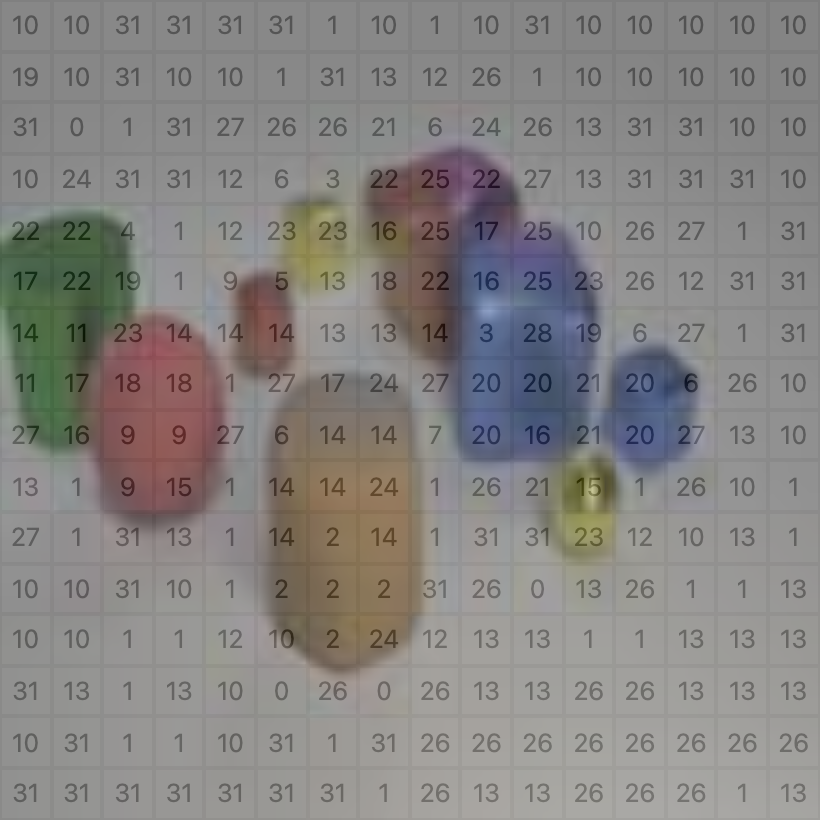}} & 
 \raisebox{15mm}{\textit{\begin{tabular}[l]{@{}l@{}}What is the size of the other object that is \\ the same material as the big \underline{purple} thing?  \\ A: Large\end{tabular}}}  &   \multicolumn{1}{c}{\includegraphics[width=0.2\textwidth]{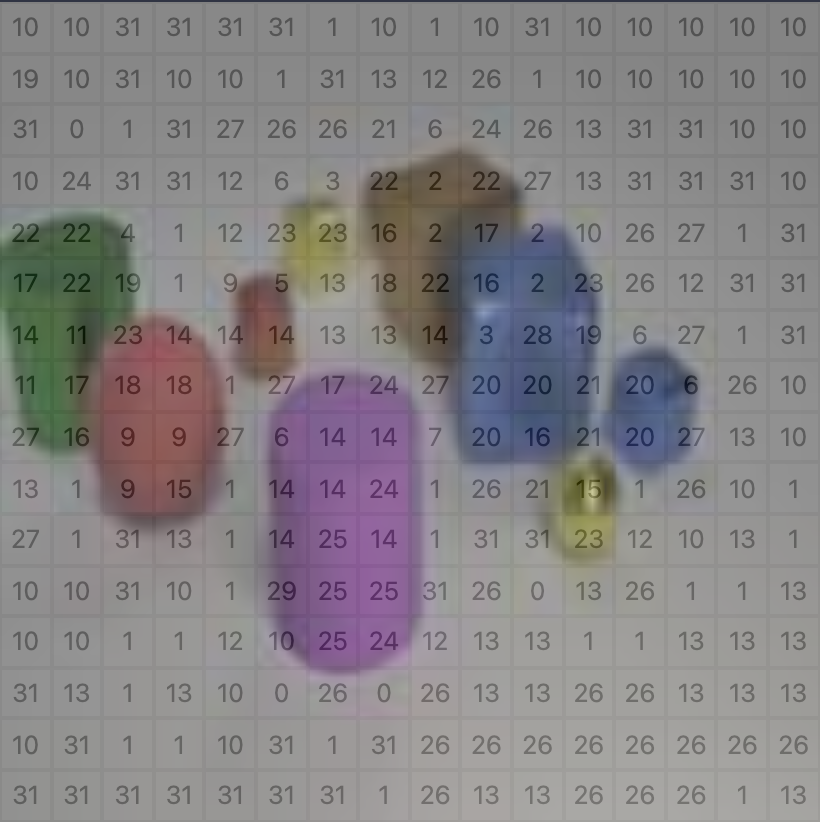}} \\

\bottomrule
\end{tabular}
}
\end{subtable}
\caption{Generated samples for \CLEVR-\CoGenT  and \COGS datasets. In  \CLEVR-\CoGenT, our method operate on displayed VQVAE symbols on top of the images and we can decode it to actual images as displayed here. The generated yellow cylinder in the first row is an unseen color+shape combination.
}
\label{tab:samplevqacogs}
\end{table*}
We present examples generated by \ourmethod in \cref{tab:samplevqacogs}. As we performed augmentation random and online during training, and we do not have a static augmented set to calculate statistics for. Instead, we run a single iteration of our augmentation function over all examples with our augmentation function and obtain following statistics:
\newline
\newline
\resizebox{\linewidth}{!}{%
\begin{tabular}{lllll}
\toprule
    Augmentation Statistics & COGS & \CLEVR & SCAN & ALCHEMY \\
\midrule
    \# Augmented samples & 24155 &  699960 & 14670 & 18285  \\
    \# Novel samples  &23301 & 548277 &   7304 &  11786  \\ 
    \# Unique novel samples & 22617 &  548277 & 4851 & 11786 \\
    \# Samples in test & 121 & 0 & 7304 & 0  \\
    \# Unique samples in test & 109 & 0 &  4851 & 0  \\
\bottomrule
\end{tabular}
}
\newline
\newline
\newline
\newline
Note that, in \CLEVR, we consider the novelty based on (question + answer) string since the generated image codes can be novel but the resulting image not. The following differences are significant under a paired t-test:
\subsection{Statistical Significance Tests for  \cref{tab:generalization_comparison}}
The following differences in \cref{tab:generalization_comparison} are significant under a paired t-test:
\paragraph{Alchemy:}
\begin{itemize}
\item T5+\ourmethod > T5 (p < 0.05)
\item LSTM+\ourmethod > {LSTM+Substitute, LSTM, LexLSTM} (p < .00001)
\end{itemize}

\paragraph{COGS:}
\begin{itemize}
\item T5+\ourmethod > T5 (p < .00001)
\item LSTM+\ourmethod  > {LSTM, \GC} (p < .00001)
\end{itemize}

\section{\CLEVR-\CoGenT Detailed Results}
\CoGenT results are presented in \cref{tab:clevr_cogent_detailed}.
\vspace{1em}
\begin{table*}[t]
\caption{Breakdown of \CLEVR-\CoGenT Results}
\label{tab:clevr_cogent_detailed}
\centering
\resizebox{\linewidth}{!}{%
\begin{tabular}{llllllll}
\toprule
  & \multicolumn{6}{c}{\bf \CLEVR-\CoGenT}  \\ 

\multicolumn{1}{l}{\textbf{VQATransformer (No Pre-Praining)}} &&&&&&\\

\boxSpace Baseline   & 73.3 \stderr{1.0}  &   71.0 \stderr{1.6} & 85.7 \stderr{0.74} & 83.5 \stderr{0.1}  & 64.4 \stderr{0.7} & 81.4 \stderr{1.2}  \\
\boxSpace + Substitute \cite[e.g.][]{liu2021counterfactual} & 84.4 \stderr{0.7}  &76.7 \stderr{1.1} & 89.5 \stderr{0.3} & 88.8 \stderr{0.3} & \bf 85.1 \stderr{1.0} & 88.0 \stderr{0.6}   \\
\boxSpace+ LexSym  & \bf 85.9 \stderr{0.9} & \bf 80.1 \stderr{0.9} & \bf 91.1 \stderr{0.5}  & \bf 91.0 \stderr{0.7} & \bf 85.2 \stderr{1.3} & \bf 88.9 \stderr{0.7}     \\
 \bottomrule
\end{tabular}
}
\end{table*}
 
\vspace{1em}
\section{Data}\label{app:data}
For \CLEVR-\CoGenT \cite{johnson2017clevr}, we use training set for Split-A as our training set, validation set for Split-B as our validation set, and validation set of Split-B as our test set. The \CLEVR and \ALCH datasets is released under the Creative Commons CC BY 4.0 license. The \COGS datasets \cite{kim2020cogs,kim2022uncontrolled} are released under MIT license. \SCAN \cite{lake2018generalization} datasets are released under BSD license. The train, validation and test set sizes are given as below.

\noindent \begin{tabular}{llll}
\toprule
    Dataset & Train  &  Validation & Test \\
\midrule
    \ALCH & 18285 & 1225 & 4495\\
    \SCAN & \\
     \sboxSpace \textit{(jump)} & 14670 & -- & 7706\\ 
    \sboxSpace \textit{(around right)} &15225 & -- &4476\\ 
    \COGS \\
      \sboxSpace \textit{(original)}&24155 & 3000 & 21000\\
      \sboxSpace \textit{(nonce)} &24155 & 3000 & 21000\\
    \CLEVR \\
      \sboxSpace \textit{(original)}& 699989 & 149991 & \\
      \sboxSpace  \textit{(CoGenT)} & 699960 & -- &150000\\
\bottomrule
\end{tabular}

\end{document}